\documentclass{article}

\usepackage[nonatbib,preprint]{neurips_2021}
\usepackage[numbers]{natbib}

\usepackage[utf8]{inputenc} 
\usepackage[T1]{fontenc}    
\usepackage{hyperref}       
\usepackage{url}            
\usepackage{booktabs}       
\usepackage{amsfonts}       
\usepackage{nicefrac}       
\usepackage{microtype}      
\usepackage{xcolor}         

\usepackage{wrapfig}
\usepackage{floatrow}
\usepackage{multirow}
\newfloatcommand{capbtabbox}{table}[][.95\FBwidth]
\usepackage{paralist}
\usepackage[normalem]{ulem} 

\usepackage{amsthm}
\usepackage[algo2e]{algorithm2e}
\providecommand{\SetAlgoLined}{\SetLine}
\providecommand{\DontPrintSemicolon}{\dontprintsemicolon}

\usepackage{graphicx}
\usepackage[]{todonotes}

\usepackage{tabularx}
\usepackage{comment}

\usepackage{subfig}
\usepackage{caption}
\usepackage{booktabs}

\newtheorem{theorem}{Theorem}
\newtheorem{lemma}[theorem]{Lemma}

\newtheorem{corollary}{Corollary}
\newtheorem{remark}{Remark}

\theoremstyle{definition}
\newtheorem{definition}{Definition}

 \newcommand{\add}[2][]{\todo[linecolor=blue,backgroundcolor=blue!25,bordercolor=blue,#1]{#2}}
 \newcommand{\Add}[2][]{\todo[inline,linecolor=blue,backgroundcolor=blue!25,bordercolor=blue,#1]{#2}}

\newcommand{\dtrain}{\mathcal{D}_{\it train}}

\usepackage{amsmath,amssymb}
\DeclareMathOperator*{\argmin}{argmin}
\DeclareMathOperator*{\argmax}{argmax}

\usepackage{bbm}
\newtheorem{assumption}{Assumption}
\newcommand{\RR}{\mathbb{R}} 
\newcommand{\EE}{\mathbb{E}} 

\newcommand{\Wcal}{\mathcal{W}}
\newcommand{\need}{{\normalfont{\text{tailor}}}}
\newcommand{\cont}{{\normalfont{\text{cont}}}}
\newcommand{\un}{{\normalfont{\text{un}}}}
\newcommand{\Lcal}{\mathcal{L}}
\newcommand{\Xcal}{\mathcal{X}}
\newcommand{\Rcal}{\mathcal{R}}
\newcommand{\Fcal}{\mathcal{F}}

\newcommand{\thetax}{{\theta_x}}
\newcommand{\thetaxi}{{\theta_{x_i}}}

\newcommand{\Ycal}{\mathcal{Y}}
\newcommand{\zetat}{\tilde \zeta}

\newcommand{\suploss}{\mathcal{L}^{\normalfont{\textrm{sup}}}}
\newcommand{\unsuploss}{\mathcal{L}^\textrm{unsup}}
\newcommand{\tailorloss}{\mathcal{L}^{\normalfont{\textrm{tailor}}}}

\newcommand{\tailor}{\tau}

\newcommand\ex[2]{#1^{(#2)}}
\newcommand{\supdata}{((x_i, y_i))_{i=1}^n}
\newcommand{\unsupdata}{(\ex{x}{j})_j}

\newcommand{\transoutput}{(\ex{\hat{y}}{j})_j}

\newcommand{\method}{{\sc CNGrad}}

\definecolor{darkolivegreen}{rgb}{0,.8,0.}
\newcommand{\greenx}{\textbf{\textcolor{darkolivegreen} x}} 

\usepackage{algorithm}
\usepackage{algpseudocode}

\title{Tailoring: encoding inductive biases by optimizing unsupervised objectives at prediction time}

\author{%
  Ferran Alet, Maria Bauza, Kenji Kawaguchi,\\ \textbf{Nurullah Giray Kuru, Tom\'as Lozano-P\'erez, Leslie Pack Kaelbling}\\
 MIT\\
  \texttt{\{alet,bauza,kawaguch,ngkuru,tlp,lpk\}@mit.edu} \\
}

\begin{document}

\maketitle
\vspace{-3mm}
\begin{abstract}
From CNNs to attention mechanisms, encoding inductive biases into neural networks has been a fruitful source of improvement in machine learning. Adding auxiliary losses to the main objective function is a general way of encoding biases that can help networks learn better representations. However, since auxiliary losses are minimized only on training data, they suffer from the same generalization gap as regular task losses. Moreover, by adding a term to the loss function, the model optimizes a different objective than the one we care about. In this work we address both problems: first, we take inspiration from \textit{transductive learning} and note that after receiving an input but before making a prediction, we can fine-tune our networks on any unsupervised loss. We call this process {\em tailoring}, because we customize the model to each input to ensure our prediction satisfies the inductive bias. Second, we formulate {\em meta-tailoring}, a nested optimization similar to that in meta-learning, and train our models to perform well on the task objective after adapting them using an unsupervised loss. The advantages of tailoring and meta-tailoring are discussed theoretically and demonstrated empirically on a diverse set of examples.~\looseness=-1
\end{abstract}
\vspace{-3mm}

\section{Introduction} \label{sec:intro}

The key to successful generalization in machine learning is the encoding of useful inductive biases.  A variety of mechanisms, from parameter tying to data augmentation, have proven useful to improve the performance of models. 
Among these, auxiliary losses 
can encode a wide variety of biases, constraints, and objectives; helping networks learn better representations and generalize more broadly. Auxiliary losses add an extra term to the task loss that is minimized over the training data. 

However, they have two major problems: 
\vspace{-2mm}
\begin{enumerate}
    \item Auxiliary losses are only minimized at training time, but not for the query points. This leads to a generalization gap between training and testing, in addition to that of the task loss.
    \item By minimizing the sum of the task loss plus the auxiliary loss, we are optimizing a different objective than the one we care about (only the task loss).
\end{enumerate}
\vspace{-2mm}
In this work we propose a solution to each problem:
\vspace{-2mm}
\begin{enumerate}
    \item We use ideas from \textit{transductive learning} to minimize unsupervised auxiliary losses at each query, thus
    eliminating their generalization gap. 
    Because these losses are unsupervised, we can optimize them at any time inside the prediction function. We call this process \textit{tailoring}, since we customize the model to each query. 
    \item We use ideas from \textit{meta-learning} to learn a model that performs well on the task loss after being tailored with the unsupervised auxiliary loss. 
    \textit{Meta-tailoring} effectively trains the model to leverage the unsupervised tailoring loss in order to minimize the task loss.
\end{enumerate} 
\vspace{-5mm}
\begin{figure*}
    \centering
    \includegraphics[width=\linewidth]{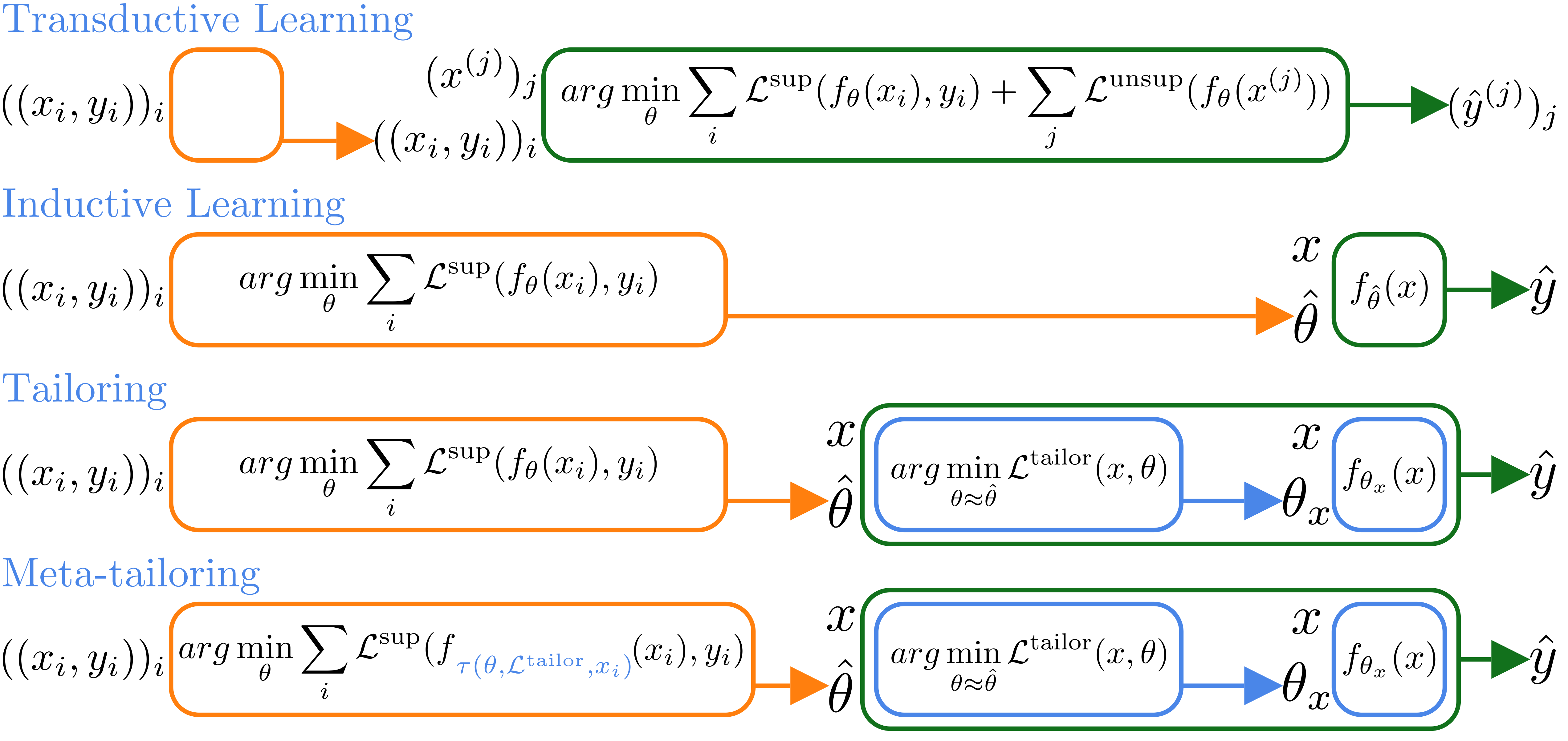}
    \caption{Comparison of several learning settings with \textit{offline}
      computation in the orange boxes and \textit{online} computation in the
      green boxes, with tailoring in blue. For meta-tailoring training, $\tau(\theta, \tailorloss,  x) = \argmin_{\theta'\approx \theta} \tailorloss(x, \theta')$ represents the tailoring process resulting in $\theta_x$. }
    \label{fig:main_diagram}
\end{figure*}
\paragraph{Tailoring a predictor}
Traditionally, supervised learning is approached within the inductive learning framework, shown in the second row of Figure ~\ref{fig:main_diagram}. There, an algorithm consumes a training dataset
of input-output pairs, $\supdata$, and produces a set of parameters
$\hat{\theta}$ by minimizing a supervised loss $\sum_{i=1}^n
\suploss(f_\theta(x_i), y_i)$ and, optionally, an unsupervised auxiliary loss $\sum_{i=1}^n\unsuploss(\theta,x_i)$. These parameters
specify a hypothesis $f_{\hat{\theta}}(\cdot)$ that, given a new input
$x$, generates an output $\hat{y} = f_{\hat{\theta}}(x)$. This problem setting misses a substantial opportunity: before the
learning algorithm sees the query point $x$, it has distilled the data
down to the parameters $\hat{\theta}$, which are frozen during inference, and so
 it cannot use new information about the \textit{particular} $x$ that it will be asked to make a prediction for.

Vapnik recognized an opportunity to make more accurate predictions
when the query point is known, in a framework that is now known as
\textit{transductive
  learning}~\citep{vapnik1995nature,chapelle2000transductive}, illustrated in the top row of Figure~\ref{fig:main_diagram}.  In
transductive learning, a single algorithm consumes both
labeled data, $\supdata$, and a set of input queries for which
predictions are desired, $\unsupdata$, and produces predicted outputs
$\transoutput$ for each query. 
In general, however, we do not know
queries \textit{a priori}, and instead, we want an inductive function that
makes predictions online, as queries arrive. 
To obtain such an online
prediction function from a transductive system, we would need to 
encapsulate the entire transductive learning procedure inside the prediction function itself.  This strategy would achieve our objective of taking $x$ into account at prediction time but would be computationally much too slow~\citep{chapelle2009semi}.~\looseness=-1 

This approach for combining induction and transduction would reuse the same training data and objective for each prediction, only changing the single unlabeled query. Consequently, it would perform extremely similar computations for each prediction.
Therefore, we propose to effectively reuse the shared computations and find a ``meta-hypothesis'' that can then be
efficiently adapted to each query. 
As shown in the third row of Figure~\ref{fig:main_diagram}, we
first run 
regular supervised learning to obtain parameters $\hat{\theta}$. Then,
given a query input $x$, we fine-tune $\hat{\theta}$ on an
unsupervised loss $\tailorloss$ to obtain customized
parameters $\theta_x$ and use them to make the final prediction:
$f_{\theta_x}(x)$.   
We call this process \textit{tailoring}, because we adapt the model to
each particular input for a customized fit. Notice that tailoring optimizes the loss at the query input, eliminating the generalization gap on the unsupervised auxiliary loss.~\looseness=-1 
\paragraph{Meta-tailoring}
Since we will be applying tailoring at prediction time, it is natural
to incorporate this adaptation during training, resulting in a
two-layer optimization similar to those used in meta-learning. Because of
this similarity, we call this process \textit{meta-tailoring}, illustrated in the bottom row
of Figure~\ref{fig:main_diagram}.  Now,
rather than letting $\hat{\theta}$ be the direct minimizer of the
supervised loss, we set it to
$$\hat{\theta} \in \text{arg}\min_\theta \sum_{i=1}^n \suploss(f_{\tailor(\theta, \tailorloss, x_i)}(x_i), y_i).$$
Here, the inner loop optimizes the unsupervised tailoring loss $\tailorloss$ and the outer loop optimizes the supervised task loss $\suploss$.
Notice that now the outer process optimizes the only objective we care, $\suploss$, instead of a proxy combination of $\suploss$ and $\unsuploss$. At the same time, we learn to leverage $\tailorloss$ in the inner loop to affect the model before making the final prediction, both during training and evaluation. Adaptation is especially clear in the case of MAML~\citep{finn2017model} when the adaptation is a step of gradient descent. We show its translation, MAMmoTh (Model-Agnostic Meta-Tailoring), in algorithm~\ref{alg:mammoth}. ~\looseness=-1

In many settings, we want to make predictions for a large
number of queries in a (mini-)batch. While MAMmoTh adapts to every input separately, it does not run in parallel efficiently for most DL frameworks.
Inspired by conditional normalization~(CN)~\citep{dumoulin2016learned} we propose \method, which adds element-wise affine transformations to our model and only adapts
the added parameters 
in the inner loop.  
This allows us to independently \textit{tailor} the model for multiple inputs in
parallel. 
We
prove theoretically, in Sec.~\ref{sec:algorithm}, and provide
experimental evidence, in Sec.~\ref{subsec:planets}, that
optimizing these parameters alone has enough capacity to minimize
a large class of tailoring losses.

\setlength{\textfloatsep}{2mm}
\begin{algorithm}[t]
\SetAlgoLined
\DontPrintSemicolon
\begin{flushleft}
  \SetKwFunction{algo}{algo}\SetKwFunction{proc}{proc}
  \SetKwProg{myalg}{Algorithm}{}{}
  \SetKwProg{myproc}{Subroutine}{}{}
  \myproc{Training($f$, $\suploss$, $\lambda_{sup}$, $\tailorloss$, $\lambda_{tailor}$, $\dtrain$,$b$)}{
  randomly initialize $\theta$\;
  
  \While{not done}{
    Sample batch of samples $(x_i,y_i) \sim \dtrain$\;
    
    \ForAll(\tcp*[f]{This loop can't be parallelized in most DL frameworks}){$(x_i,y_i)$}{
        
        $\theta_{x_i} = \theta - \lambda_{tailor}\nabla_{\theta}\tailorloss(\theta,x_i)$\; \tcp*[r]{Inner step with tailor loss}
        
    }
    $\theta = \theta - \lambda_{sup}\nabla_\theta \sum_{(x_i,y_i)} \suploss\left(f_{\theta_{x_i}}(x_i), y_i\right)$\; \tcp*[r]{Outer step with supervised loss}
  }
  \Return $\theta$\;
  }
  \end{flushleft}
\caption{\textbf{MAM}mo\textbf{T}h: \textbf{M}odel-\textbf{A}gnostic \textbf{M}eta-\textbf{T}ailoring \label{alg:mammoth}}
\end{algorithm}

\paragraph{Relation between (meta-)tailoring, fine-tuning transfer, and meta-learning}

Fine-tuning pre-trained networks is a fruitful method of transferring knowledge from large corpora to smaller related datasets~\citep{donahue2014decaf}. Fine-tuning allows reusing features on related tasks or for different distributions of the same task. When the data we want to adapt to is unlabeled, we must use unsupervised losses. This can be useful to adapt to changes of task~\citep{dhillon2019baseline}, from simulated to real data~\citep{wu2017marrnet}, or to new  distributions~\citep{sun2019test}.~\looseness=-1 

Tailoring performs unsupervised fine-tuning and is, in this sense, similar to test-time training(TTT)~\citep{sun2019test} for a single sample, which adapts to distribution shifts.  However, tailoring is applied to a single query; not to a data set that captures distribution shift, where batched TTT sees most of its benefits. Thus, whereas fine-tuning benefits from more adaptation data, tailoring is hindered by more data. This is because tailoring aims at building a custom model for each query to ensure the network satisfies a particular inductive bias. Customizing the model to multiple samples makes it harder, not easier. We show this in Figure~\ref{fig:planet-results}, where TTT with 6400 samples performs worse than tailoring with a single sample. Furthermore, tailoring adapts to each query one by one, not globally from training data to test data. Therefore, it also makes sense to do tailoring on training queries (i.e. meta-tailoring). 

Meta-tailoring has the same two-layer optimization structure as meta-learning. More concretely, it can be understood as the extreme case of meta-learning where each single-query prediction is its own task.
However, whereas meta-learning tasks use one loss and different examples for the inner and outer loop, meta-tailoring tasks use different losses and one example for each loop~($\tailorloss,\suploss$). 
We emphasize that meta-tailoring does not operate in the typical multi-task meta-learning setting. 
Instead, we are leveraging techniques from meta-learning for the classical single-task ML setting.
    \paragraph{Contributions}
In summary, our contributions are:~\looseness=-1
\begin{enumerate} \vspace{-6pt}
    \item Introducing \textit{tailoring}, a new framework for encoding inductive biases by minimizing unsupervised losses at prediction time, with theoretical guarantees and broad potential applications.~\looseness=-1
    \item Formulating \textit{meta-tailoring}, which adjusts the outer objective to optimize only the task loss,
    and developing a new algorithm, \method, for efficient meta-tailoring. ~\looseness=-1
    \item Demonstrating \textit{tailoring} in 3 domains: encoding
    hard and soft conservation laws in physics prediction problems~(Sec.~\ref{subsec:planets} and~Sec.~\ref{subsec:soft}), enhancing
    resistance to adversarial examples by increasing local
    smoothness at prediction time~(Sec.~\ref{sec:adv}), and improving prediction quality both theoretically~(Sec.~\ref{sec:contrastive_theory}) and empirically~(Sec.~\ref{sec:contrastive})
    by tailoring with a contrastive loss.~\looseness=-1
\end{enumerate}

\vspace{-4mm}
\section{Related work} \label{sec:related_work} \vspace{-6pt}
Tailoring is inspired by transductive learning. However, transductive methods, because they operate on a batch of unlabeled queries, are allowed to make use of the underlying distributional properties of those queries, as in semi-supervised learning~\cite{chapelle2009semi}.  
In contrast, tailoring does the bulk of the computations before receiving any query; vastly increasing efficiency. 
Similar to tailoring, local learning~\citep{bottou1992local} 
also has input-dependent parameters. However, it uses similarity in raw input space to select a few labeled data points and builds a local model instead of reusing the global prior learned across the whole data. 
Finally, some methods~\citep{garcia2017few,liu2018learning} in meta-learning propagate predictions along the test samples in a semi-supervised transductive fashion.


Similar to tailoring, there are other learning frameworks that perform optimization at prediction time for very different purposes. Among those, energy-based models do generative modeling~\citep{ackley1985learning,hinton2002training,lecun2006tutorial} by optimizing the hidden activations of neural networks,
 and other models~\citep{amos2017optnet,tschiatschek2018differentiable} learn to solve optimization problems by embedding optimization layers in neural networks. 
 In contrast, tailoring 
 instead optimizes the parameters of the model, not the hidden activations or the output.

As discussed in the introduction, unsupervised fine-tuning methods 
have been proposed to adapt to different types of variations between training and testing. ~\citet{sun2019test} propose to adapt to a change of distribution with few samples by unsupervised fine-tuning at test-time, applying it with a loss of predicting whether the input has been rotated. ~\citet{zhang2020adaptive} build on it to adapt to group distribution shifts with a learned loss. 
%
Other methods in the few-shot meta-learning setting exploit test samples of a new task by minimizing either entropy~\citep{dhillon2019baseline} or a learned loss~\citep{antoniou2019learning} in the inner optimization. Finally, ~\citet{wang2019dynamic} use entropy in the inner optimization to adapt to large-scale variations in image segmentation. 
In contrast, we propose (meta-)tailoring as a general effective way to impose inductive biases in the classic machine learning setting. Whereas in the aforementioned methods adaptation happens from training to testing, we independently adapt to every single query. 

Meta-learning~\citep{schmidhuber1987evolutionary,bengio1995search,thrun2012learning,hospedales2020meta} has the same two-level optimization structure as meta-tailoring but focuses on multiple prediction tasks. 
As shown in Alg.~\ref{alg:mammoth} for MAML~\citep{finn2017model}, most optimization-based meta-learning algorithms can be converted to meta-tailoring. 
Similar to \method, there are other meta-learning methods whose adaptations can be batched~\citep{rakelly2019efficient,alet2019neural}. 
Among these,~\citep{zintgraf2018fast,requeima2019fast} train FiLM networks~\citep{perez2018film} to predict custom conditional normalization~(CN) layers for each task. By optimizing the CN layers directly, \method\ is simpler, while remaining provably expressive~(section~\ref{sec:algorithm}). CNGrad can also start from an trained model by initializing the CN layers to the identity function. 
\vspace{-4mm}
\section{Theoretical motivations of meta-tailoring} \label{sec:theory_meta} 
\vspace{-4mm}
In this section, we study the potential advantages of meta-tailoring from the theoretical viewpoint, formalizing the intuitions conveyed in the introduction. By acting symmetrically during training and prediction time, meta-tailoring allows us to closely relate its training and expected losses, whereas tailoring, in general, may make them less related. First, we analyze the particular case of a contrastive tailoring loss. Then, we will generalize the guarantees to other types of tailoring losses. 
\vspace{-3mm}
\subsection{Meta-tailoring with a contrastive tailoring loss}\label{sec:contrastive_theory} \vspace{-3mm}
\textit{Contrastive learning}~\citep{hadsell2006dimensionality} has seen significant successes in problems of semi-supervised learning~\citep{oord2018representation,he2019momentum,chen2020simple}. The main idea is to create multiple versions of each training image and learn a representation in which  variations of the same image are close while variations of different images are far apart. 
Typical augmentations involve cropping, color distortions, and rotation. We show theoretically that, under reasonable conditions, meta-tailoring using a particular contrastive loss $\Lcal_{\cont}$ as $\tailorloss=\Lcal_{\cont}$ helps us improve generalization errors in expectation compared with performing classical inductive learning.~\looseness=-1


When using meta-tailoring, we define $\theta_{x,S}$ to be the $\theta_x$ obtained with a training dataset $S=((x_i,y_i))_{i=1}^n$ and tailored with the contrastive loss at the prediction point $x$. 
Theorem \ref{thm:cont_2} provides an upper bound on the expected supervised loss $\EE_{x,y}[\suploss(f_{\theta_{x,S}}(x),y)]$ in terms of the expected contrastive loss $\EE_{x}[\Lcal_{\cont}(x,\theta_{x,S})]$ (defined and analyzed in App.~\ref{sec:contrastive_def}), the empirical supervised loss $\frac{1}{n}\sum_{i=1}^{n}\suploss(f_{\theta_{x_i,S}}(x_i),y_i)$ of meta-tailoring, and its uniform stability $\zeta$.
%
Theorem \ref{thm:cont_1} (App.~\ref{sec:all_proofs}) provides a similar bound with the Rademacher complexity  \citep{bartlett2002rademacher}  $\Rcal_n(\suploss \circ \Fcal)$ of the set  $\suploss \circ \Fcal$, instead of using the uniform stability $\zeta$,. 
Proofs of all results in this paper are deferred to App. \ref{sec:all_proofs}.~\looseness=-1

\begin{definition} \label{def:stability}
Let $S=((x_i,y_i))_{i=1}^n$ and $S'=((x_i',y_i'))_{i=1}^n$ be any two training datasets that differ by a single point. Then, a meta-tailoring algorithm $S \mapsto  f_{\theta_{\greenx,S}}(x)$ is \textit{uniformly $\zeta$-stable} if  
$
\forall (x,y) \in \Xcal \times \Ycal, \ |\suploss(f_{\theta_{\greenx,S}}(x),y)-\suploss(f_{\theta_{\greenx,S'}}(x),y)| \le \frac{\zeta}{n}.
$ 
\end{definition}
\begin{theorem} \label{thm:cont_2}
Let $S \mapsto  f_{\theta_{\greenx,S}}(x)$ be a uniformly $\zeta$-stable meta-tailoring algorithm.
Then, for any $\delta>0$, with probability at least $1-\delta$ over  an i.i.d. draw of $n$ i.i.d. samples  $S=((x_i, y_i))_{i=1}^n$, the following holds: for any $\kappa \in [0, 1]$,
$
\EE_{x,y}[\suploss(f_{\theta_{\greenx,S}}(x),y)] \le \kappa \EE_{x}\left[\Lcal_{\cont}^{}(x,\theta^{}_{\greenx,S}) \right]+ (1-\kappa)\mathcal{J}, 
$
where
$
\mathcal{J} = \frac{1}{n}\sum_{i=1}^{n}\suploss(f_{\theta_{\greenx_i,S}}(x_i),y_i)+\frac{\zeta}{n}+(2\zeta+c) \sqrt{(\ln(1/\delta))/(2n)},
$ 
and $c$ is the upper bound on the per-sample loss as $\suploss(f_{\theta}(x), y) \le c$.
\vspace{-2mm}
\end{theorem}

 In the case of regular inductive learning, we get a bound of the exact same form, except that we have a single $\theta$ instead of a $\theta_\greenx$ tailored to each input $x$. 
 This theorem illustrates the effect of meta-tailoring on contrastive learning, with its potential reduction of the expected contrastive loss $\EE_{x}[\Lcal_{\cont}(x,\theta_{x,S})]$. In classic induction, we may aim to minimize the empirical contrastive loss $\frac{1}{\bar n}\sum_{i=1}^{\bar n}\Lcal_{\cont}(x_{i},\theta)$ with $\bar n$ potentially unlabeled training samples, 
 which incurs the additional generalization error of  $\EE_{x}[\Lcal_{\cont}(x,\theta_{x,S})]-\frac{1}{\bar n}\sum_{i=1}^{\bar n}\Lcal_{\cont}(x_{i},\theta)$. 
 In contrast, meta-tailoring can avoid this extra generalization error by directly minimizing $\EE_{x}[\Lcal_{\cont}(x,\theta_{x,S})]$.

In the case where $\EE_{x}[\Lcal_{\cont}(x,\theta_{x,S})]$ is left large (e.g., due to large computational cost
), Theorem \ref{thm:cont_2} still illustrates competitive generalization bounds of meta-tailoring with small $\kappa$. 
For example, with $\kappa=0$, it provides generalization bounds with the uniform stability for meta-tailoring algorithms. Even then, the bounds are not equivalent to those of classic induction, and there are potential benefits of meta-tailoring, which are discussed in the following section with a more general setting. 

\vspace{-4mm}
\subsection{Meta-tailoring with general tailoring losses}\label{subsec:theory-motivation} \vspace{-3mm}
The benefits of meta-tailoring go beyond contrastive learning: below we provide generalization bounds for meta-tailoring with any tailoring loss $\tailorloss(x, \theta)$ and any supervised loss $\suploss(f_\theta(x),y)$.~\looseness=-1

\begin{remark} \label{coro:gen_1}
For any function $\varphi$ such that $\EE_{x,y}[\suploss(f_{\theta}(x),y)] \le\EE_{x}[\varphi (\tailorloss(x, \theta))]$, Theorems \ref{thm:cont_2} and \ref{thm:cont_1}  hold with the map $\Lcal_{\cont}$ being replaced by the function $\varphi \circ \tailorloss$.  
\vspace{-2mm}
\end{remark}

This remark shows the benefits of meta-tailoring through its effects on three factors: the expected unlabeled loss $\EE_{x}[\varphi (\tailorloss(x, \theta_{x,S}))]$, uniform stability $\zeta$, and the Rademacher complexity $\Rcal_{n}(\suploss \circ \Fcal)$. It is important to note that meta-tailoring can directly minimize the expected unlabeled loss $\EE_{x}[\varphi (\tailorloss(x, \theta_{x,S}))]$, whereas classic induction can only minimize its empirical version, which results in the additional generalization error on the difference between the expected unlabeled loss and its empirical version. 
For example, if $\varphi$ is monotonically increasing and $\tailorloss(x, \theta)$ represents the physical constraints at each input $x$ (as in the application in section \ref{subsec:planets}), then  classic induction requires the physical constraints of neural networks at the  \textit{training} points to generalize to the physical constraints at \textit{unseen} (e.g., testing) points.
Meta-tailoring avoids this requirement by directly minimizing violations of the physical constraints at each point at prediction time.

Meta-tailoring can also improve the \textit{parameter stability} $\zeta_\theta$ defined such that 
$
\forall (x,y) \in \Xcal \times \Ycal, \| \theta_{x,S}-\theta_{x,S'}\| \le \frac{\zeta_{\theta}}{n},
$ for all $S, S'$ differing by a single point.
When $ \theta_{x,S}=\hat{\theta}_S - \lambda\nabla\tailorloss(x, \hat{\theta}_S)$
, we obtain an improvement on the parameter stability $\zeta_\theta$ if  $\nabla\tailorloss(x, \hat{\theta}_S)$ can pull $\hat{\theta}_S $ and $\hat{\theta}_{S'}$ closer so that
$
\| \theta_{x,S}-\theta_{x,S'} \|<\|\hat{\theta}_S -\hat{\theta}_{S'}\|,
$
which is ensured, for example, if $\|\cdot\|=\|\cdot \|_2$ and $\text{cos\_dist}(v_1,v_2)\frac{\|v_{1}\|}{ \|v_{2}\|}>\frac{1}{2}$ where $\text{cos\_dist}(v_1,v_2)$ is the cosine similarity of $v_1$ and $v_2$, with $v_1=\hat{\theta}_S -\hat{\theta}_{S'}$,  $v_2=\lambda(\nabla\tailorloss(x, \hat{\theta}_S)-\nabla\tailorloss(x, \hat{\theta}_{S'}))$ and $v_2\neq0$. 
Here, the uniform stability $\zeta$ and the parameter stability $\zeta_\theta$ are closely related as $\zeta \le C\zeta_\theta$, where $C$ is the upper bound on the Lipschitz constants of the maps $\theta \mapsto \suploss(f_{\theta}(x),y)$ over all $ (x,y) \in \Xcal \times \Ycal$ under the norm $\|\cdot\|$,  since $|\suploss(f_{\theta_{x,S}}(x),y)-\suploss(f_{\theta_{x,S'}}(x),y)| \le C \| \theta_{x,S}-\theta_{x,S'}\| \le  \frac{C\zeta_\theta}{n}$.



\vspace{-4mm}
\section{\method: a simple algorithm for expressive, efficient (meta-)tailoring}\label{sec:algorithm} \vspace{-2mm}
In this section, we address the issue of using (meta-)tailoring for efficient GPU computations. Although possible in JAX~\cite{jax2018github}, efficiently parallelizing the evaluation of different tailored models is not possible in other frameworks. 
To overcome this issue, building on CAVIA~\citep{zintgraf2018fast} and WarpGrad~\citep{flennerhag2019meta}, we propose \method\ which adapts only \textit{conditional normalization} parameters and enables efficient GPU computations for (meta-)tailoring. 
\method\ can also be used in meta-learning~(see App. ~\ref{app:method}).

As is done in batch-norm~\citep{ioffe2015batch} after element-wise normalization, we can implement an element-wise affine transformation with parameters $(\gamma, \beta)$, scaling and shifting the output $h^{(l)}_k(x)$ of each $k$-th neuron at the $l$-th hidden layer independently: $\gamma^{(l)}_k h^{(l)}_k(x) + \beta^{(l)}_k$. 
In conditional normalization,~\citet{dumoulin2016learned} train a collection of $(\gamma,\beta)$ in a multi-task fashion to learn different tasks with a single network. \method\  brings this concept to meta-learning and (meta-)tailoring settings and adapts the affine parameters $(\gamma,\beta)$ to each query. 
For meta-tailoring, the inner loop minimizes the tailoring loss at an input $x$ by adjusting the affine parameters and the outer optimization adapts the rest of the network. Similar to MAML~\citep{finn2017model}, we implement a first-order version, which does not backpropagate through the optimization, and a second-order version, which does.
\method\  efficiently parallelizes computations of multiple tailored models 
because the adapted parameters only require element-wise multiplications and additions. See Alg.~\ref{alg:maintext_CNGRAD} for the pseudo-code.

\setlength{\textfloatsep}{2mm}
\begin{algorithm}[t]
\SetAlgoLined
\DontPrintSemicolon
\begin{flushleft}
  \SetKwFunction{algo}{algo}\SetKwFunction{proc}{proc}
  \SetKwProg{myalg}{Algorithm}{}{}
  \SetKwProg{myproc}{Subroutine}{}{}
  \myproc(\tcp*[f]{Only in meta-tailoring}){Training($f$, $\suploss$, $\lambda_{sup}$, $\tailorloss$, $\lambda_{tailor}$, $steps$,$\dtrain$,$b$)}{
  randomly initialize $w$ \tcp*[r]{All parameters except $\gamma,\beta$; trained in outer loop}
  \While{not done}{
    $X,Y \sim^b \dtrain; \gamma_0 = \mathbf{1}_{b,\sum_l m_l}; \beta_0 = \mathbf{0}_{b,\sum_l m_l}$\;\tcp*[r]{Sample batch; initialize $\gamma,\beta$}
    
    \For{$1\leq s \leq steps$}{
        $\gamma_s = \gamma_{s-1} - \lambda_{tailor}\nabla_{\gamma}\tailorloss(w,\gamma_{s-1},\beta_{s-1},X)$\; \tcp*[r]{Inner step w.r.t. $\gamma$}

        $\beta_s = \beta_{s-1} - \lambda_{tailor}\nabla_{\beta}\tailorloss(w,\gamma_{s-1},\beta_{s-1},X)$\; \tcp*[r]{Inner step w.r.t. $\beta$}
        
        $\gamma_s,\beta_s = \gamma_s.detach(), \beta_s.detach()\;$ \tcp*[r]{Only in $1^{st}$ order CNGrad}
        
        $grad_w = grad_w+\nabla_w \suploss\left(f_{w,\gamma_s,\beta_s}(X), Y\right)$\; \tcp*[r]{Outer gradient w.r.t. $w$}
        
    }
    $w = w - \lambda_{sup}grad_w$\;\tcp*[r]{Apply outer step after all inner steps}
  }
  \Return $w$\;
  }

  \myproc(\tcp*[f]{Both in meta-tailoring \& tailoring}){Prediction($f$, $w$, $\tailorloss$, $\lambda$, $steps$, $X$)}{
    $\gamma_0 = \mathbf{1}_{X.shape[0],\sum_l m_l}; \beta_0 = \mathbf{0}_{X.shape[0],\sum_l m_l}$\;
    
    \For{$1\leq s \leq steps$}{
        $\gamma_s = \gamma_{s-1} - \lambda\nabla_{\gamma}\tailorloss(w,\gamma_{s-1},\beta_{s-1},X)$\;

        $\beta_s = \beta_{s-1} - \lambda\nabla_{\beta}\tailorloss(w,\gamma_{s-1},\beta_{s-1},X)$\;
    }
  \Return $f_{w,\gamma_{steps},\beta_{steps}}(X)$\;
  }
  \end{flushleft}
\caption{\method \label{alg:maintext_CNGRAD} for meta-tailoring}
\vspace{-0.25cm}
\end{algorithm}

\method\ is widely applicable since the adaptable affine parameters can be added to any hidden layer, and only represent a tiny portion of the network (empirically, $1\%$). Moreover, we can see that, under realistic assumptions, we can minimize the inner tailoring loss using only the affine parameters. 
To analyze properties of these adaptable affine parameters, let us decompose $\theta$  into $\theta=(w,\gamma,\beta)$, where $w$ contains all the weight parameters (including bias terms), and the $(\gamma,\beta)$ contains all the affine parameters. 
Given an arbitrary function $(f_{\theta}(x),x) \mapsto \ell_{\need}(f_{\theta}(x),x)$, 
let $\tailorloss(x,\theta)= \sum_{i=1}^{n_{g}}\ell_{\need}(f_{\theta}(g^{(i)}(x)),x)$, 
where  $g^{(1):(n_g)}$ are arbitrary input augmentation functions at prediction time. 
Note that $n_g$ is typically small ($n_g \ll n$) in meta-tailoring.

Corollary \ref{corollary:expresivity} states that for any given $\hat w$, if we add any non-degenerate Gaussian noise $\delta$ as $\hat w + \delta$ with zero mean and any variance on $ \delta$, the global minimum value of $\tailorloss$ w.r.t. all parameters $(w,\gamma,\beta)$ can be achieved by  optimizing only the affine parameters $(\gamma,\beta)$, with probability one. 

\begin{corollary} \label{corollary:expresivity}
Under the assumptions of Theorem \ref{thm:expresivity}, for any $\hat w \in \RR^d$, with  probability one over  randomly sampled   $\delta \in \RR^d$ accordingly to any non-degenerate Gaussian distribution, the following holds: $\inf_{w,\gamma,\beta} \tailorloss(x,w,\gamma,\beta)=\inf_{\gamma,\beta} \tailorloss(x,\hat w +\delta ,\gamma,\beta)$ for any $x \in \Xcal$.  
\vspace{-2mm}
\end{corollary}
The assumption and condition in theorem~\ref{thm:expresivity} are satisfied in practice~(see App.~\ref{sec:g_condition_explain}). Therefore, \method\ is a practical and computationally efficient method to implement (meta-)tailoring.~\looseness=-1

\vspace{-3mm}
\section{Experiments}\vspace{-1mm}

\subsection{Tailoring to impose symmetries and constraints at prediction time} \label{subsec:planets}
\vspace{-1mm}
Exploiting invariances and symmetries is an established strategy for increasing performance in ML. 
During training, we can regularize networks to satisfy certain criteria; but this doesn't guarantee they will be satisfied outside the training dataset~\citep{suh2020surprising}. (Meta-)tailoring provides a general solution to this problem by adapting the model at prediction time to satisfy the criteria.
We demonstrate the use of tailoring to enforce physical conservation laws for predicting the evolution of a 5-body planetary system. This prediction problem is challenging, as $m$-body systems become chaotic for $m>2$. We generate a dataset with positions, velocities, and masses of all 5 bodies as inputs and the changes in position and velocity as targets.  App.~\ref{app:planets} further describes the dataset.~\looseness=-1
\begin{figure}
\begin{floatrow}
\capbtabbox{
\begin{tabular}{lcr}
    \toprule[1.5pt]
      \textbf{Method}& loss & relative\\
    \midrule[2pt]
    Inductive learning & .041 &- \\
    Opt. output(50 st.) & .041& (0.7 $\pm$ 0.1)\% \\ 
    6400-s. TTT(50 st.) & .040 & (3.6 $\pm$ 0.2)\% \\
    \midrule
    Tailoring(1 step) & .040& (1.9 $\pm$ 0.2)\% \\
    Tailoring(5 steps) & .039 & (6.3 $\pm$ 0.3)\% \\
    Tailoring(10 st.) & .038 & (7.5  $\pm$ 0.1)\%\\
    Meta-tailoring(0 st.) & .030 & (26.3  $\pm$ 3.3)\%\\ 
    Meta-tailoring(1 st.) & .029 & (29.9  $\pm$ 3.0)\%\\
    Meta-tailoring(5 st.) & .027 & (35.3  $\pm$ 2.6)\%\\
    Meta-tailoring(10 s.) & .026 & (36.0  $\pm$ 2.6)\%\\
    \bottomrule[1.5pt]
  \end{tabular}
}
{
  \caption{Test MSE loss for different methods;  the second column shows the relative improvement over basic inductive supervised learning. The test-time training (TTT) baseline uses a full batch of 6400 test samples to adapt, not allowed in regular SL. With a few gradient steps, tailoring significantly over-performs all baselines. Meta-tailoring improves even further, with $35\%$ improvement.
  }
  \label{table:planets}
}
\ffigbox{
  \includegraphics[width=\linewidth]{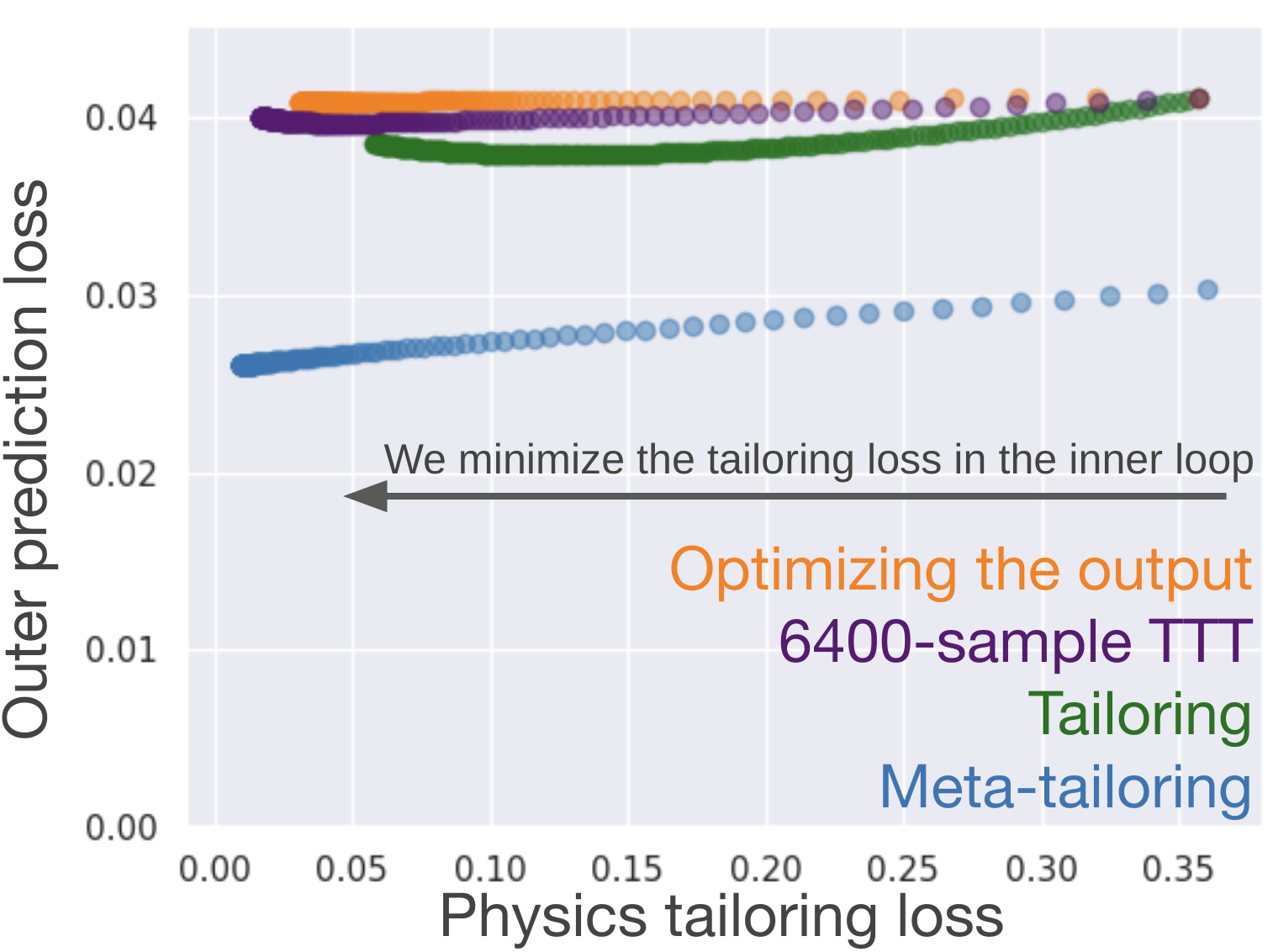}%
  }
  {%
  \caption{
  Optimization at prediction time on the planet data; each path going from right to left as we minimize the physics tailoring loss.  We use a small step size to illustrate the path.
     Tailoring and the two baselines only differ in their test-time computations, thus sharing their starts. Meta-tailoring has a lower starting loss, faster optimization, and no overfitting during tailoring.~\looseness=-1
  \label{fig:planet-results}}%
}
\end{floatrow}
\vspace{-.0cm}
\end{figure}

Our model is a 3-layer feed-forward network. We tailor it by taking the original predictions and adapting the model using the tailoring loss given by the $L_1$ loss between the initial and final energy and momentum of the whole system. Note that ensuring this conservation doesn't guarantee better performance: predicting the input as the output conserves energy and momentum perfectly, but is not correct.~\looseness=-1


While tailoring adapts some parameters in the network to improve the tailoring loss, an alternative for enforcing conservation would be to adapt the output $y$ value directly.  
Table~\ref{table:planets} compares the predictive accuracy of inductive learning, direct output optimization and both tailoring and meta-tailoring, using varying numbers of gradient steps.
Tailoring is more effective than adapting the output, as the parameters provide a prior on what changes are more natural.
For meta-tailoring, we try both first-order and second-order versions of \method: the first-order gave slightly better results, possibly because it was trained with a higher tailor learning rate ($10^{-3}$) with which the second-order version was unstable~(we thus used $10^{-4}$). More details can be found in App.~\ref{app:planets}.~\looseness=-1 

Finally, meta-tailoring without any query-time tailoring steps already performs much better than the original model, even though both have almost the same number of parameters and can overfit the dataset.  
We conjecture meta-tailoring training adds an inductive bias that guides optimization towards learning a more generalizable model. Fig.~\ref{fig:planet-results} shows prediction-time optimization paths.~\looseness=-1


\vspace{-2mm}
\subsection{Tailoring to softly encourage inductive biases} \label{subsec:soft}
\vspace{-1mm}

\begin{wrapfigure}{r}{0.5\linewidth}
 \centering
    \vspace{-3mm}
    \includegraphics[width=\linewidth]{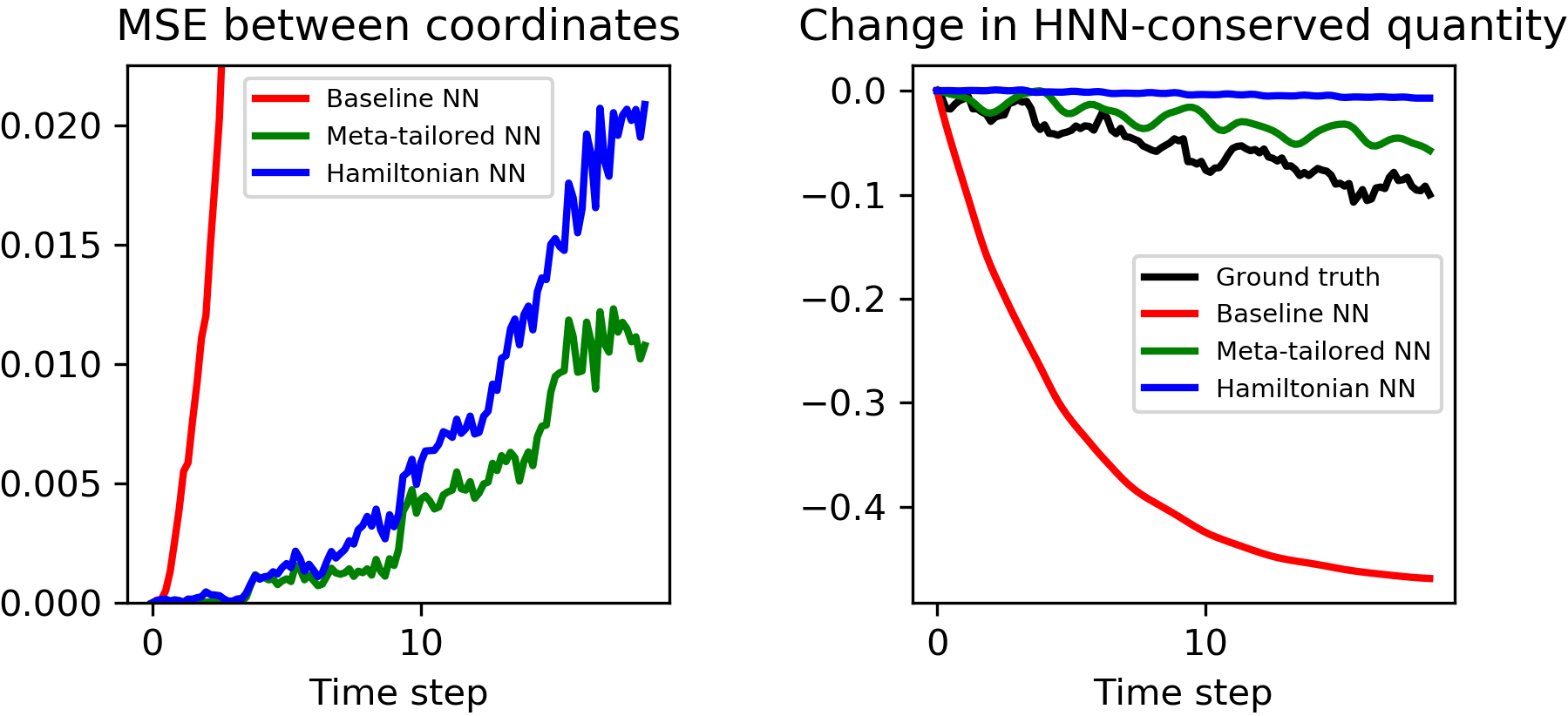}
    \vspace{-0mm}
    \caption{By softly encouraging energy conservation, meta-tailoring improves over models that don't and models that fully impose it.}
    \label{fig:soft_pendulum}
    \vspace{-7mm}
\end{wrapfigure}

A popular way of encoding inductive biases is with clever network design to make predictions translation equivariant~(CNNs), permutation equivariant~(GNNs), or conserve energy~\citep{greydanus2019hamiltonian}. 
However, if an inductive bias is only partially satisfied, such approaches overly constrain the function class. Tailoring instead can softly impose this bias by only fine-tuning the corresponding tailoring loss for a few steps.

We showcase this in the real pendulum experiment used by Hamiltonian Neural Networks~(HNNs)~\citep{greydanus2019hamiltonian}. 
HNNs have energy conservation built-in, and easily improve a vanilla MLP. We meta-tailor this vanilla MLP with energy conservation, without changing its architecture. Meta-tailoring significantly improves over the baseline and HNNs, since it can encode the \textit{imperfect} energy conservation of real systems.
We compare results in Fig.~\ref{fig:soft_pendulum} and provide extra details in App.~\ref{app:soft_pendulum}.
Note that, with inexact losses, fully enforcing them provides sub-optimal results. Thus, we pick the tailoring learning rate that results in the lowest long-term prediction loss during training.~\looseness=-1

\vspace{-2mm}
\subsection{Tailoring with a contrastive loss for image classification}\label{sec:contrastive}

\begin{wrapfigure}{r}{0.4\linewidth}
 \centering
    \includegraphics[width=\linewidth]{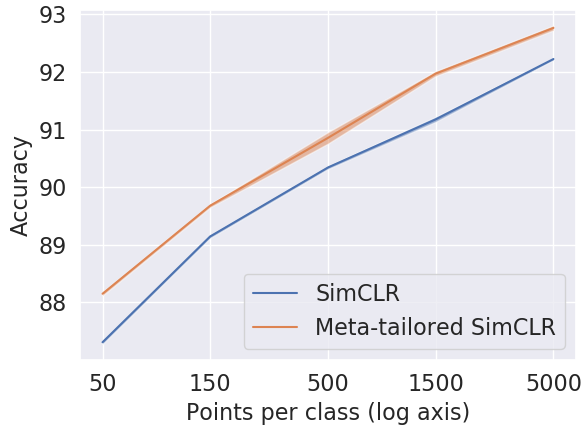}
    
    \caption{Meta-tailoring the linear layer with the contrastive loss results in consistent gains between $0.5\%$ and $0.8\%$ in accuracy. This is approximately the same gain as that of doubling the amount of labeled data (note the logarithmic x-axis).}
    \label{fig:contrastive_results}
    \vspace{-4mm}
\end{wrapfigure}

Following the setting described in section~\ref{subsec:theory-motivation}, we provide experiments on the CIFAR-10 dataset~\citep{krizhevsky2009learning} by building on SimCLR~\citep{chen2020simple}. SimCLR trains a ResNet-50~\citep{he2016deep} $f_\theta(\cdot)$ coupled to a small MLP $g(\cdot)$ such that the outputs of two augmentations of the same image $x_i,x_j\sim\mathcal{T}(x)$ agree; i.e. $g(f_\theta(x_i))\approx g(f_\theta(x_j))$. 
This is done by training $g(f(\cdot))$ to recognize one augmentation from the other among a big batch of candidates with the cross-entropy loss. To show that the unsupervised training of $f_\theta$ provides a useful representation, SimCLR trains a single linear layer on top of it, $\phi(f_\theta(\cdot))$, achieving good classification results.

Building on SimCLR, tailoring $f_{\theta}$ at prediction-time by optimizing $g(f_{\theta_x}(x))$ maximizes agreement between augmentations. 
At train time, meta-tailoring doesn't redo SimCLR unsupervised learning, which provides $\theta$. Its outer loop only trains $\phi$ to take the tailored representations $f_{\theta_x}(x)$. Thus, $\theta$ is unsupervisedly fine-tuned in the prediction function leading to $\theta_x$, but never supervisedly trained as this would break the evaluation protocol~(in meta-tailoring's favor). 
Results in Fig.~\ref{fig:contrastive_results}, TTT~\citep{sun2019test} performed worse than base SimCLR~(see App.).\looseness=-1

\vspace{-2mm}
\subsection{Tailoring for robustness against adversarial examples}\label{sec:adv}

Neural networks are susceptible to adversarial examples~\citep{biggio2013evasion,szegedy2013intriguing}: targeted small perturbations of an input
can cause the network to misclassify it. 
One approach is to make the prediction function smooth via adversarial training~\citep{madry2017towards}; however, this only ensures smoothness in the training points. Constraining the model to be smooth everywhere makes it lose capacity.
Instead, (meta-)tailoring asks for smoothness \textit{a posteriori}, only on a specific query.~\looseness=-1

\begin{table}
\small
\centering
\begin{tabular}{c|l|rrrrrrr|r}
    \toprule[1.5pt]
      $\sigma$ & \textbf{Method}& 0.0 & 0.5& 1.0 & 1.5  & 2.0 & 2.5 & 3.0 & ACR\\
    \midrule[2pt]
    \multirow{2}{*}{0.25} & (Inductive) Randomized Smoothing & 0.67 & 0.49 & 0.00 & 0.00 & 0.00 & 0.00 & 0.00 & 0.470\\
     & Meta-tailored Randomized Smoothing & \textbf{0.72} & \textbf{0.55} & 0.00 & 0.00 & 0.00 & 0.00 & 0.00 & \textbf{0.494}\\
    \midrule[1pt]
    \multirow{2}{*}{0.50}  & (Inductive) Randomized Smoothing& 0.57 & 0.46 & 0.37 & 0.29 & 0.00 & 0.00 & 0.00 & 0.720\\
     & Meta-tailored Randomized Smoothing & 0.66 & 0.54 & \textbf{0.42} & \textbf{0.31} & 0.00 & 0.00 & 0.00 & \textbf{0.819}\\
    \midrule[1pt]
    \multirow{2}{*}{1.00} & (Inductive) Randomized Smoothing & 0.44 & 0.38 & 0.33 & 0.26 & 0.19 & 0.15 & 0.12 & 0.863\\
    & Meta-tailored Randomized Smoothing & 0.52 & 0.45 & 0.36 & \textbf{0.31} & \textbf{0.24} & \textbf{0.20} & \textbf{0.15} & \textbf{1.032}\\
    \bottomrule[2pt]
  \end{tabular}
    \vspace{-2mm}
  \caption{Percentage of points with certificate above different radii, and ACR for ImageNet. Meta-tailoring improves ACR by $5.1\%,13.8\%,19.6\%$. Results for~\citet{cohen2019certified} are taken from~\citep{Zhai2020MACER}.\looseness=-1
  \label{table:imagenet}
  }
\end{table}~\vspace{-5mm}

We apply meta-tailoring to robustly classifying CIFAR-10~\citep{krizhevsky2009learning} and ImageNet~\citep{deng2009imagenet} images, tailoring predictions so that they are locally smooth. Inspired by the notion of adversarial examples being caused by predictive, but non-robust, features~\citep{ilyas2019adversarial}, we meta-tailor our model 
by enforcing smoothness on the vector of features of the penultimate layer~(denoted $g_\theta(x)$):
\vspace{-1mm}
$$
\tailorloss(x,\theta)  = \mathbb{E}[\text{cos\_dist}(g_\theta(x),  g_\theta(x+\delta))], 
\delta\sim N(0,\nu^2),
$$
\vspace{-4mm}

We build on~\citet{cohen2019certified}, who developed a method for certifying the robustness of a model via randomized smoothing~(RS).  RS samples points from a Gaussian $N(x,\sigma^2)$ around the query and, if there is enough agreement in classification, it provides a certificate that the query cannot be adversarially modified by a small perturbation to have a different class.  
We show that meta-tailoring improves the original RS method, testing for $\sigma=0.25,0.5,1.0$. We use $\nu = 0.1$ for all experiments. We initialized with the weights of~\citet{cohen2019certified} by leveraging that \method\ can start from a pre-trained model by initializing the extra affine layers to the identity.
Finally, we use $\sigma' = \sqrt{\sigma^2-\nu^2}\approx 0.23,0.49,0.995$ so that the points used in our tailoring loss come from $N(x,\sigma^2)$.~\looseness=-1 

Table~\ref{table:cifar} shows our results on CIFAR-10 where we improve the average certification radius~(ARC) by $8.6\%,10.4\%,19.2\%$ respectively. In table~\ref{table:imagenet}, we show results on Imagenet where we improve the ARC by $5.1\%,13.8\%,19.6\%$ respectively. We chose to meta-tailor the RS method 
because it represents a strong standard in certified adversarial defenses, but we note that there have been advances on RS that sometimes achieve better results than those presented here~\citep{Zhai2020MACER,salman2019provably}, see App.~\ref{app:adv}. However, likely, meta-tailoring can also improve these methods.~\looseness=-1 


These experiments only scratch the surface of what tailoring allows for adversarial defenses: usually, the adversary looks at the model and gets to pick a particularly bad perturbation $x+\delta$. With tailoring, the model responds, by changing to weights $\theta_{x+\delta}$. This leads to a game, where both weights and inputs are perturbed, similar to $\max_{|\delta|<\epsilon_x}\min_{|\Delta|<\epsilon_\theta} \suploss\left( f_{\theta+\Delta}(x+\delta), y\right)$. However, since we don't get to observe $y$; we optimize the weight perturbation by minimizing $\tailorloss$ instead.

\vspace{-6pt}
\section{Discussion}\label{sec:conclusion} \vspace{-6pt}
\subsection{Broader Impact}~\label{subsec:impact}
\uline{Improving adversarial robustness:} having more robust and secure ML systems is mostly a positive change. However, improving adversarial defenses could also go against privacy preservation. \uline{Encoding desirable properties:} By optimizing an unsupervised loss for the particular query we care about, it is easier to have guarantees on the prediction. In particular, there could be potential applications for fairness, where the unsupervised objective could enforce specific criteria at the query or related inputs. More research needs to be done to make this assertion formal and practical. \quad
\uline{Potential effect on privacy:} tailoring specializes the model to each input. This could  have an impact on privacy. Intuitively, the untailored model can be less specialized to each input, lowering the individual information from each training point contained in the model. However, tailored predictions have more information about the queries, from which more personal information could be leaked. 


\vspace{-6pt}
\subsection{Limitations}\label{sec:limitations} \vspace{-6pt}
Tailoring provides a framework for encoding a wide array of inductive biases, but these need to be specified as a formula by the user. For instance, it would be hard to codify tailoring losses in raw pixel data, such as mass conservation in pixel space. Tailoring also incurs an extra time cost at prediction time, since we make an inner optimization inside the prediction function. However, as shown in Table~\ref{table:planets}, meta-tailoring often achieves better results than inductive learning even without adaptation at test-time, enabling predictions at regular speed during test-time. This is due to meta-tailoring leading to better training. Moreover, optimization can be sped up by only tailoring the last layers, as discussed in App.~\ref{app:method}. Finally, to the best of our knowledge a naive implementation of tailoring would be hard to parallelize in PyTorch~\citep{paszke2019pytorch} and Tensorflow~\citep{abadi2016tensorflow}; we proposed \method~ to make it easy and efficient. JAX\citep{jax2018github}, which better handles per-example weights, makes parallelizing tailoring effortless.~\looseness=-1

Theory in Sec.~\ref{sec:theory_meta} applies only for meta-tailoring. Unlike tailoring (and test-time training), meta-tailoring performs the same computations at training and testing time, which allows us to prove the results. Theorem~\ref{thm:expresivity} proves that optimizing the CN layers in~\method~ has the same expressive power as optimizing all the layers for the inner (not outer) loss. However, it does not guarantee that gradient descent will find the appropriate optima. The study of such guarantee is left for future work.~\looseness=-1
\vspace{-3mm}
\subsection{Conclusion}
We have presented \textit{tailoring}, a simple way of embedding a powerful class of inductive biases into models, by minimizing unsupervised objectives at prediction time. Tailoring leverages the generality of auxiliary losses and improves them in two ways: first, it eliminates the generalization gap on the auxiliary loss by optimizing it on the query point; second, tailoring only minimizes task loss in the outer optimization and the tailoring loss in the inner optimization. This results in the whole network optimizing the only objective we care about, instead of a proxy loss. Finally, we have formalized these intuitions by proving the benefits of meta-tailoring under mild assumptions.

Tailoring is broadly applicable, as one can vary the model, the unsupervised loss, and the task loss. 
We show its applicability in three diverse domains: physics prediction time-series, contrastive learning, and adversarial robustness. We also provide a simple algorithm,~\method, to make meta-tailoring practical with little additional code.
Currently, most unsupervised or self-supervised objectives without taking into account the supervised down-stream task. 
Instead, meta-tailoring provides a generic way to make these objectives especially useful for each particular application. 

\begin{ack}
We would like to thank Kelsey Allen, Marc de la Barrera, Jeremy Cohen, Dylan Doblar, Chelsea Finn, Sebastian Flennerhag, Jiayuan Mao, and Shengtong Zhang for insightful discussions. We would also like to thank Clement Gehring for his help with deploying the experiments and Lauren Milechin for her help with leveraging the MIT supercloud platform~\citep{reuther2018interactive}.

We gratefully acknowledge support from NSF grant 1723381; from AFOSR grant FA9550-17-1-0165; from ONR grant N00014-18-1-2847; from the Honda Research Institute, from MIT-IBM Watson Lab; and from SUTD Temasek Laboratories. We also acknowledge the MIT SuperCloud and Lincoln Laboratory Supercomputing Center for providing HPC resources that have contributed to the reported research results. Any opinions, findings, and conclusions or recommendations expressed in this material are those of the authors and do not necessarily reflect the views of our sponsors.~\looseness=-1
\end{ack}
\medskip

{
\small
\bibliography{main}
\bibliographystyle{icml2021}
}

\section*{Checklist}
\begin{enumerate}

\item For all authors...
\begin{enumerate}
  \item Do the main claims made in the abstract and introduction accurately reflect the paper's contributions and scope?
    \answerYes{}
  \item Did you describe the limitations of your work?
    \answerYes{In section~\ref{sec:limitations}.}
  \item Did you discuss any potential negative societal impacts of your work?
    \answerYes{In section~\ref{subsec:impact}.}
  \item Have you read the ethics review guidelines and ensured that your paper conforms to them?
    \answerYes{}
\end{enumerate}

\item If you are including theoretical results...
\begin{enumerate}
  \item Did you state the full set of assumptions of all theoretical results?
    \answerYes{Important assumptions are in the main text; all assumptions are detailed in the appendix, particularly in appendix~\ref{sec:g_condition_explain}.}
	\item Did you include complete proofs of all theoretical results?
    \answerYes{In appendix~\ref{sec:all_proofs}.}
\end{enumerate}

\item If you ran experiments...
\begin{enumerate}
  \item Did you include the code, data, and instructions needed to reproduce the main experimental results (either in the supplemental material or as a URL)?
    \answerNo{We plan, however, to open-source our codebase once cleaned.}
  \item Did you specify all the training details (e.g., data splits, hyperparameters, how they were chosen)?
    \answerYes{Distributed accross multiple sections in the appendix.}
	\item Did you report error bars (e.g., with respect to the random seed after running experiments multiple times)?
    \answerNo{We report them for the planet experiments and the contrastive experiments. The adversarial experiments are extremely computationally expensive and we only ran them once.}
	\item Did you include the total amount of compute and the type of resources used (e.g., type of GPUs, internal cluster, or cloud provider)?
    \answerYes{In each relevant appendix.}
\end{enumerate}

\item If you are using existing assets (e.g., code, data, models) or curating/releasing new assets...
\begin{enumerate}
  \item If your work uses existing assets, did you cite the creators?
    \answerYes{}
  \item Did you mention the license of the assets?
    \answerYes{In the appendix.}
  \item Did you include any new assets either in the supplemental material or as a URL?
    \answerNA{}
  \item Did you discuss whether and how consent was obtained from people whose data you're using/curating?
    \answerNA{}
  \item Did you discuss whether the data you are using/curating contains personally identifiable information or offensive content?
    \answerNA{}
\end{enumerate}

\item If you used crowdsourcing or conducted research with human subjects...
\begin{enumerate}
  \item Did you include the full text of instructions given to participants and screenshots, if applicable?
    \answerNA{}
  \item Did you describe any potential participant risks, with links to Institutional Review Board (IRB) approvals, if applicable?
    \answerNA{}
  \item Did you include the estimated hourly wage paid to participants and the total amount spent on participant compensation?
    \answerNA{}
\end{enumerate}

\end{enumerate}

\newpage
\appendix

\section{Theorem \ref{thm:expresivity}, Corollary \ref{corollary:expresivity} and interpretation of their conditions} \label{sec:g_condition_explain}

\begin{assumption}\label{a:activation}
\textit{(Common activation)}
The activation function $\sigma(x)$ is real analytic, monotonically  increasing, and the limits exist as: $\lim_{x\rightarrow -\infty}\sigma(x)=\sigma_->-\infty$ and $\lim_{x\rightarrow +\infty}\sigma(x)=\sigma_+\leq+\infty$.
\end{assumption}

\begin{theorem} \label{thm:expresivity}
 For any $x \in \Xcal $ that satisfies $\|g^{(i)}(x)\|_2^2- g^{(i)}(x)^{\top} g^{(j)}(x) > 0$ (for all $ i\neq j$), and for any fully-connected neural network with a single output unit, at least $n_{g}$ neurons per hidden layer, and activation functions that satisfy Assumption \ref{a:activation}, the following holds: 
$\inf_{w,\gamma,\beta} \tailorloss(x,w,\gamma,\beta)=\inf_{\gamma,\beta} \tailorloss(x,\bar w,\gamma,\beta)$ for  any $\bar w \notin \Wcal$ where  Lebesgue measure of $\Wcal \subset \RR^d$ is zero.
\end{theorem}

Assumption \ref{a:activation} is satisfied by using common activation functions such as sigmoid and hyperbolic tangent, as well as {\em softplus}, which is defined as $\sigma_\alpha(x)=\ln(1+\exp(\alpha x))/\alpha$ and satisfies Assumption \ref{a:activation} with any hyperparameter $\alpha \in \RR_{>0}$. The softplus activation function can approximate the ReLU function to any desired accuracy: i.e.,
$
\sigma_{\alpha}(x) \rightarrow \mathrm{relu}(x) \text{ as } \alpha\rightarrow \infty,
$
where $\mathrm{relu}$ represents the ReLU function. 


In Theorem \ref{thm:expresivity} and Corollary \ref{corollary:expresivity}, the condition $\|g^{(i)}(x)\|_2^2- g^{(i)}(x)^{\top} g^{(j)}(x) > 0$ (for all $ i\neq j$) can be easily satisfied, for example,  by choosing $g^{(1)},\dots,g^{(n_g)}$ to produce  normalized  and distinguishable argumented inputs for each prediction point  $x$ at prediction time. To see this, with normalization  $\|g^{(i)}(x)\|^{2}_2=\|g^{(j)}(x)\|_2^2$, the condition is satisfied if $\|g^{(i)}(x)-g^{(j)}(x)\|_2^2>0$ for $i\neq j$   since  $\frac{1}{2}\|g^{(i)}(x)-g^{(j)}(x)\|_2^2=\|g^{(i)}(x)\|_2^{2}- g^{(i)}(x)^{\top} g^{(j)}(x)$. 

In general, the normalization is not necessary for  the condition to hold; e.g., orthogonality on   $g^{(i)}(x)$ and $g^{(j)}(x)$ along with $g^{(i)}(x) \neq 0$ satisfies it without the  normalization.

\section{Understanding the expected meta-tailoring contrastive loss} \label{sec:contrastive_def}
To analyze meta-tailoring for contrastive learning, we focus on the binary classification loss of the form $\suploss(f_\theta(x),y)=\ell_{\cont}(f_\theta(x)_{y}-f_\theta(x)_{y'=\neg y})$ where $\ell_{\cont}$ is convex and  $\ell_{\cont}(0)=1$. With this, the objective function $\theta\mapsto \suploss(f_\theta(x),y)$ is still non-convex in general. For example, the standard hinge loss $\ell_{\cont}(z)=\max\{0,1-z\}$ and the logistic loss $\ell_{\cont}(z)=s\log_{2}(1+\exp(z))$ satisfy this condition.

We first define the meta-tailoring contrastive loss  $\Lcal_{\cont}(x,\theta)$ in detail. In meta-tailoring contrastive learning, we  choose the probability measure of positive example $x^+ \sim \mu_{x^{+}}(x)$ and  the probability measure of negative example $x^-,y^- \sim \mu_{x^{-},y^{-}}(x)$,
both of which are tailored for  each input point $x$ at prediction time. 
These choices induce the marginal distributions   for the negative examples  $y^{-}\sim\mu_{y^{-}}(x)$ and $x^- \sim \mu_{x^{-}}(x)$, as well as the unknown probability of $y^{-}= y$ defined by $\rho_{y}(\mu_{y^{-}}(x))=\EE_{y^- \sim\mu_{y^{-}}(x)}(\mathbbm{1}\{y^{-}= y\})$. 
Define the lower and upper bound on the probability of $y^{-}= y$  as $\underline{\rho}(x)  \le \rho_{y}(\mu_{y^{-}}(x))\le  \bar \rho(x)\in [0,1) $.

Then, the first pre-meta-tailoring contrastive loss can be defined by$$
\Lcal_{\cont}^{x^{+},x^-}(x,\theta)=\EE_{\substack{x^+ \sim \mu_{x^{+}}(x), \\ x^-\sim\mu_{x^{-}}(x)}}[\ell_{\cont}(h_{\theta^{}}(x)^{\top}(h_{\theta^{}}(x^{+})-h_{\theta^{}}(x^{-})))],
$$
where   $h_{\theta^{}}(x) \in \RR^{m_H+1}$  represents the output of the last hidden layer, including  a constant neuron corresponding the bias term of the last output layer  (if there is no bias term, $h_{\theta^{}}(x) \in \RR^{m_H}$).
For every  $z \in \RR^{2 \times (m_{H}+1)}$,
define
$
\psi_{x,y,y^- }(z)=\ell_{\cont}((z_{y}-z_{y^{-}})h_{\theta^{}}(x)),
$
where $z_{y}\in \RR^{1 \times m_{H}}$ is the $y$-th row vector of $z$. We define the second pre-meta-tailoring contrastive loss  by 
$$
\Lcal_{\cont}^{x^{+},x^{-},y^- }(x,\theta)=\max_{y}\EE_{y^- \sim\mu_{y^{-}}(x)}[\psi_{x, y,y^{-}}(\theta^{(H+1)})-\psi_{x,1,2}([u_h^+,u_h^-]^{\top})],
$$
where $u_h^+=\EE_{x^+ \sim \mu_{x^{+}}(x)}[h_{\theta^{}}(x^+ )]$ and $u_h^-=\EE_{x^-\sim\mu_{x^{-}}(x)}[h_{\theta^{}}(x^-)]$. Here, we decompose $\theta$  as $\theta=(\theta^{(1:H)},\theta^{(H+1)})$, where  $\theta^{(H+1)}=[W^{(H+1)},b^{(H+1)}] \in \RR^{m_y \times( m_H+1)}$ represents the parameters at the last output layer, and   $\theta^{(1:H)}$ represents all others. 

Then, the meta-tailoring contrastive loss is defined by 
$$
\Lcal_{\cont}(x,\theta) =\frac{1}{1- \bar \rho(x)  }\left(\Lcal_{\cont}^{x^{+},x^-}(x,\theta)+\Lcal_{\cont}^{x^{+},x^{-},y^- }(x,\theta)-\underline{\rho}(x)\right).
$$

Theorem \ref{thm:cont_3} states that for any $\theta^{(1:H)}$,  the convex optimization of $\Lcal_{\cont}^{x^{+},x^-}(x,\theta)+\Lcal_{\cont}^{x^{+},x^{-},y^- }(x,\theta)$ over $\theta^{(H+1)}$ can achieve the value of $\Lcal_{\cont}^{x^{+},x^-}(x,\theta)$ without the value of $\Lcal_{\cont}^{x^{+},x^{-},y^- }(x,\theta)$, allowing us to focus  on the first term $\Lcal_{\cont}^{x^{+},x^-}(x,\theta)$, for some choice of  $\mu_{x^{-},y^{-}}(x)$ and $\mu_{x^{+}}(x)$. 

\begin{theorem} \label{thm:cont_3}
 For any $\theta^{(1:H)},\mu_{x^{-},y^{-}}(x)$ and $\mu_{x^{+}}(x)$, the function $\theta^{(H+1)} \mapsto\Lcal_{\cont}^{x^{+},x^-}(x,\theta)+\Lcal_{\cont}^{x^{+},x^{-},y^- }(x,\theta)$ is convex.
Moreover, there exists  $\mu_{x^{-},y^{-}}(x)$ and $\mu_{x^{+}}(x)$ such that, for any $\theta^{(1:H)}$ and any $\bar \theta^{(H+1)}$,
$$
\inf_{\theta^{(H+1)}\in \RR^{m_y \times( m_H+1)}} \Lcal_{\cont}^{x^{+},x^-}(x,\theta)+\Lcal_{\cont}^{x^{+},x^{-},y^- }(x,\theta)\le \Lcal_{\cont}^{x^{+},x^-}(x,\theta^{(1:H)}, \bar \theta^{(H+1)}) .
$$ 
\end{theorem}

\section{Proofs} \label{sec:all_proofs}
In order to have concise proofs, we introduce additional notations while keeping track of dependent variables more explicitly. Since $h_\theta$ only depends on $\theta^{(1:H)}$, let us write $h_{\theta^{(1:H)}}=h_\theta$. Similarly, $\Lcal_{\cont}(x,\theta^{(1:H)} )=\Lcal_{\cont}(x,\theta)$. Let $\theta(x)=\theta_x$ and $\theta(x,S)=\theta_{x,S}$. Define $\Lcal = \suploss$.

\subsection{Proof of Theorem \ref{thm:expresivity}}

\begin{proof}[Proof of Theorem \ref{thm:expresivity}]
The output of fully-connected neural networks  for an input $x$ with a parameter vector $\theta=(w,\gamma,\beta)$ can be represented by  $f_{\theta}(x)=W^{(H+1)}h^{(H)}(x)+b^{(H+1)}$ where $W^{(H+1)}\in \RR^{1 \times m_{H}}$ and $b^{(H+1)}\in \RR$ are the weight matrix and the bias term respectively  at the last layer, and $h^{(H)}(x) \in \RR^{m_{H}}$ represents the output of the last hidden layer. 
Here, $m_l$ represents the number of neurons at the $l$-th layer, and  $h^{(l)}(x)=\gamma^{(l)}(\sigma(W^{(l)}h^{(l-1)}(x)+b^{(l)}))-\beta^{(l)} \in \RR^{m_{l}}$ for $l=1,\dots,H$, with trainable parameters $\gamma^{(l)},\beta^{(l)} \in \RR^{m_l}$, where  $h^{(0)}(x)=x$. Let $z ^{(l)}(x)=\sigma(W^{(l)}h^{(l-1)}(x)+b^{(l)})$.

 Then, by rearranging the definition of the output of the neural networks,  
\begin{align*}
f_{\theta}(x)&=W^{(H+1)}h^{(H)}(x)+b^{(H+1)}
\\ & = \left( \sum_{k=1}^{m_H} W_{k}^{(H+1)}\gamma_{k}^{(H)}   z^{(H)}(x)_{k}  +W_{k}^{(H+1)}\beta^{(H)}_{k} \right)+b^{(H+1)} 
\\ & =[W^{(H+1)} \circ  z^{(H)}(x)^{\top}, W^{(H+1)} ]\begin{bmatrix}\gamma^{(H)} \\
\beta^{(H)} \\
\end{bmatrix} +b^{(H+1)}.
\end{align*}
Thus, we can write
\begin{align} \label{proof:eq2}
\begin{bmatrix}
f_{\theta}(g^{(1)}(x)) \\
\vdots \\
f_{\theta}(g^{(n_{g})}(x)) \\
\end{bmatrix} =M_{w} \begin{bmatrix}\gamma^{(H)} \\
\beta^{(H)} \\
\end{bmatrix} + b^{(H+1)} \mathbf{1}_{n_{g}} \in \RR^{n_{g}},
\end{align}
where
$$
M_{w} =\begin{bmatrix}W^{(H+1)} \circ  z^{(H)}(g^{(1)}(x))^{\top}, W^{(H+1)}  \\
\vdots \\
W^{(H+1)} \circ z^{(H)}(g^{(n_{g})}(x))^{\top}, W^{(H+1)} \\
\end{bmatrix}\in \RR^{n_{g} \times 2m_H}, 
$$
and $\mathbf{1}_{n_{g}}=[1,1,\dots,1]^{\top} \in \RR^{n_{g}}$.

Using the above equality, we show an exitance of a $(\gamma,\beta)$ such that $\tailorloss(x,\bar w,\gamma,\beta)=\inf_{w,\gamma,\beta} \tailorloss(x,\theta)$ for any $x \in \Xcal \subseteq \RR^{m_x}$ and  any $\bar w \notin \Wcal$ where  Lebesgue measure of $\Wcal \subset \RR^d$ is zero.
To do so, we first fix $\gamma^{(l)}_k=1$ and $\beta^{(l)}_k = 0$ for $l=1,\dots,H-1$, with which $h^{(l)}(x)=z ^{(l)}(x)$  for $l=1,\dots,H-1$.

Define $\varphi(w)=\det(M_wM_w^{\top})$, which is   analytic  since $\sigma$ is analytic.
Furthermore, we have that $\{w\in \RR^{d}: \text{$M_{w}$ has rank less than ${n_{g}}$} \}
=\{w\in \RR^{d}:\varphi(w)=0\},
$
since the rank of $M_{w}$ and the rank of the Gram matrix are equal.
Since $\varphi$ is analytic, if $\varphi$ is not identically zero ($\varphi\neq 0$),  the Lebesgue measure of its zero set 
$
\{w\in \RR^{d}:\varphi(w)=0\}
$
is zero \citep{mityagin2015zero}. 
Therefore, if   $\varphi(w)\neq 0$ for some $w \in \RR^d$, the  Lebesgue measure of the set $\{w\in \RR^{d}: \text{$M_{w}$ has rank less than ${n_{g}}$} \}$ is zero. 

Accordingly, we now constructs a $w \in \RR^d$ such that $\varphi(w)\neq 0$. Set $W^{(H+1)}=\mathbf{1}_{m_H}^{\top}$. Then,
$$
M_{w} = [\bar M_w, \mathbf{1}_{n_{g},m_H}] \in \RR^{n_{g} \times m_H}. 
$$ 
where
$$
\bar M_w =\begin{bmatrix}z^{(H)}(g^{(1)}(x))^{\top} \\
\vdots \\
z^{(H)}(g^{(n_{g})}(x)(x))^{\top} \\
\end{bmatrix} \in \RR^{n_{g} \times m_H} 
$$
and $\mathbf{1}_{n_{g},m_H} \in \RR^{n_{g} \times m_H}$ with $(\mathbf{1}_{n_{g},m_H})_{ij}=1$ for all $i,j$. For $l=1,\dots,H$,
 define 
$$
G^{(l)} = \begin{bmatrix}z ^{(l)}(g^{(1)}(x))^{\top} \\
\vdots \\
 z ^{(l)}(g^{({n_{g}} )}(x))^{\top} \\
\end{bmatrix} \in \RR^{{n_{g}} \times m_{l}}.
$$
Then, for $l=1,\dots,H$,
$$
G^{(l)} = \sigma(G^{(l-1)} (W^{(l)})^\top + \mathbf{1}_{n_{g}}(b^{(l)})^{\top}),
$$
where $\sigma$ is applied element-wise (by  overloading  of the notation $\sigma)$, and 
$$
(\bar M_w)_{ik} =(G^{(H)} )_{ik} .
$$
From the assumption $g(x)$, there exists $c>0$ such that $\|g^{(i)}(x)\|_2^2-\langle g^{(i)}(x), g^{(j)}(x)\rangle > c$ for all $i\neq j$. From Assumption \ref{a:activation}, there exists $c'$ such that $\sigma_{+} - \sigma_- >c'$. Using these constants,
set $W^{(1)}_i=\alpha^{(1)} g^{(i)}(x)^{\top}$ and $b^{(1)}_i=c\alpha^{(1)}/2 - \alpha^{(1)} \|g^{(i)}(x)\|_2^2$ for $i=1,\dots,n_{g}$, where $W^{(1)}_i$ represents the $i$-th row  of $W^{(1)}$.
Moreover, set $W^{(l)}_{1:n_{g},1:n_{g}}=\alpha^{(l)} I_{n_{g}}$ and $b^{(l)}_k=c'\alpha^{(l)}/2 -\alpha^{(l)} \sigma_+$ for all $k$ and $l=2,\dots,H$, where $W^{(l)}_{1:n_{g},1:n_{g}}$ is the fist $n_{g} \times n_{g}$ block matrix of $W^{(1)}$ and $I_{n_{g}}$ is the  $n_{g} \times n_{g}$ identity matrix.
Set all other weights and bias to be zero. Then,
for any $i\in\{1,\dots,n_{g}\}$,
$$
(G^{(1)} )_{ii} = \sigma(c\alpha^{(1)}/2), 
$$
and for any $k\in\{1,\dots,n_{g}\}$ with  $k \neq i$,  
$$
(G^{(1)} )_{ik} =\sigma(\alpha^{(1)} (\langle g^{(i)}(x), g^{(k)}(x)\rangle- \|g^{(k)}(x)\|_2^{2}+c/2) ) \le\sigma(-c\alpha^{(1)}/2).  
$$
Since $\sigma(c\alpha^{(1)}/2)\rightarrow\sigma_{+}$ and $\sigma(-c\alpha^{(1)}/2) \rightarrow \sigma_-$  
as  $\alpha^{(1)}\rightarrow \infty$, with  $\alpha^{(1)}$ sufficiently large, we have that $\sigma(c\alpha^{(1)}/2)-\sigma_++c'/2\ge c^{(2)}_1$ and $\sigma(-c\alpha^{(1)}/2)-\sigma_++c'/2 \le -c^{(2)}_2 $ for some $c^{(2)}_1,c^{(2)}_2>0$. Note that $c^{(2)}_1$ and $c^{(2)}_2$ depends only on $\alpha^{(1)}$ and does not depend on any of $\alpha^{(2)},\dots,\alpha^{(H)}$. Therefore, with  $\alpha^{(1)}$ sufficiently large,
$$
(G^{(2)})_{ii} =\sigma(\alpha^{(2)} (\sigma(c\alpha^{(1)}/2)-\sigma_++c'/2) ) \ge\sigma(\alpha^{(2)} c^{(2)}_1) ,
$$
and  
$$
(G^{(2)})_{ik} \le \sigma(\alpha^{(2)}(\sigma(-c\alpha^{(1)}/2)-\sigma_++c'/2)  ) \le\sigma(-\alpha^{(2)} c^{(2)}_2)   . 
$$
Repeating this process with Assumption \ref{a:activation}, we have that with    $\alpha^{(1)},\dots,\alpha^{(H-1)}$ sufficiently large,
$$(G^{(H)})_{ii} \ge\sigma(\alpha^{(H)} c^{(H)}_1),$$ 
and 
$$(G^{(H)})_{ik} \le\sigma(-\alpha^{(H)} c^{(H)}_2).
$$ 
Here,  $(G^{(H)})_{ii}\rightarrow\sigma_{+}$ and $(G^{(H)})_{ik}\rightarrow\sigma_{-}$ as $\alpha^{(H)} \rightarrow \infty$.
Therefore, with    $\alpha^{(1)},\dots,\alpha^{(H)}$ sufficiently large,
for any $i \in \{1,\dots,n_{g}\}$,
\begin{align} \label{proof:eq1}
\left|(\bar M_w)_{ii} - \sigma_{-}  \right| > \sum_{k\neq i} \left|(\bar M_w)_{ik}  - \sigma_{-}  \right|.
\end{align}
The inequality \eqref{proof:eq1} means that the matrix $\bar M'_w=[(\bar M_w)_{ij} - \sigma_{-}]_{1\le i,j\le n_{g}} \in \RR^{n_{g} \times n_g}$
 is strictly diagonally dominant and hence is nonsingular with rank $n_{g}$. This implies that the matrix $[\bar M'_w, \mathbf{1}_{n_g}]\in \RR^{n_{g} \times (n_g+1)}$ has rank $n_g$. This then implies that the matrix $\tilde M_w=[[(\bar M_w)_{ij}]_{1\le i,j\le n_{g}}, \mathbf{1}_{n_g}]\in \RR^{n_{g} \times (n_g+1)}$ has rank $n_g$, since the elementary matrix operations preserve the matrix rank. Since the set of all columns of $M_w$ contains all columns of $\tilde M_w$, this implies that $M_w$ has rank $n_{g}$ and $\varphi(w)\neq 0$ for this constructed particular $w$.  

Therefore, the Lebesgue measure of the  set 
$
\Wcal = \{w\in \RR^{d}:\varphi(w)=0\}
$
is zero.
If $w \notin \Wcal$, $\{(f_{\bar w,\bar \gamma,\bar \beta}(g^{(1)}(x)),\dots,f_{\bar w,\bar \gamma,\bar \beta}(g^{(n_{g})}(x))\in \RR^{n_g}:\bar \gamma^{(l)},\bar \beta ^{(l)}\in \RR^{ml}\}=\RR^{n_g}$, since $M_w$ has rank $n_{g}$ in \eqref{proof:eq2} for some $\bar \gamma^{(l)},\bar \beta ^{(l)}$ for $l=1,\dots,H-1$ as shown above. Thus, for any $\bar w \notin \Wcal$ and for any $(w,\gamma,\beta)$,  there exists $(\bar \gamma,\bar \beta)$ such that 
$$
(f_{w,\gamma,\beta}(g^{(1)}(x)),\dots,f_{w,\gamma,\beta}(g^{(n_{g})}(x)) =(f_{\bar w,\bar \gamma,\bar \beta}(g^{(1)}(x)),\dots,f_{\bar w,\bar \gamma,\bar \beta}(g^{(n_{g})}(x)) 
$$
which implies the desired statement.  

\end{proof}

\subsection{Proof of Corollary \ref{corollary:expresivity}}
\begin{proof}[Proof of Corollary \ref{corollary:expresivity}]
Since non-degenerate Gaussian measure with any mean and variance is absolutely continuous with respect to Lebesgue measure, Theorem \ref{thm:expresivity} implies the statement of this corollary.

\end{proof}

\subsection{Proof of Theorem \ref{thm:cont_2}}

The following lemma provides an upper bound on the expected loss via expected meta-tailoring contrastive loss.

\begin{lemma} \label{lemma:contrastive}
For every $\theta$,
\begin{align*}  
\EE_{x,y}[\Lcal(f_\theta(x),y)]\le 
\EE_{x}\left[\frac{1}{1- \bar \rho(x)  }\left(\Lcal_{\cont}^{x^{+},x^-}(x,\theta^{(1:H)})+\Lcal_{\cont}^{x^{+},x^{-},y^- }(x,\theta^{})-\underline{\rho}(x)\right) \right]   
\end{align*}
\end{lemma} 
\begin{proof}[Proof of Lemma \ref{lemma:contrastive}]
Using the notation $\rho=\rho_{y}(\mu_{y^{-}}(x))$,
\begin{align*}
&\EE_{x,y}[\Lcal(f_\theta(x),y)] 
\\ & =\EE_{x,y}\left[\frac{1}{1-\rho}\left((1-\rho)\Lcal(f_\theta(x),y) \pm \rho \right) \right]
\\ & =\EE_{x,y}\left[\frac{1}{1-\rho}\left((1-\rho) \ell_{\cont}(f_\theta(x)_{y}-f_\theta(x)_{y^{-}\neq y}) +\rho \ell_{\cont}(f_\theta(x)_{y}-f_\theta(x)_{y^{-}= y}) - \rho\right) \right] 
\\ & =\EE_{x,y}\left[\frac{1}{1-\rho}\left(\EE_{y^- \sim\mu_{y^{-}}(x)}[\ell_{\cont}(f_\theta(x)_{y}-f_\theta(x)_{y^{-}}) ]- \rho\right) \right] 
\\ & =\EE_{x,y}\left[\frac{1}{1-\rho}\left(\EE_{y^- \sim\mu_{y^{-}}(x)}[\psi_{x, y,y^{-}}(\theta^{(H+1)})]- \rho\right) \right] 
\\ & \le \EE_{x,y}\left[\frac{1}{1-\rho}\left(\psi_{x,1,2}([u_h^+,u_h^-]^{\top})+\Lcal_{\cont}^{x^{+},x^{-},y^- }(x,\theta^{})- \rho\right) \right]
\\ & \le \EE_{x,y}\left[\frac{1}{1-\rho}\left(\Lcal_{\cont}^{x^{+},x^-}(x,\theta^{(1:H)})+\Lcal_{\cont}^{x^{+},x^{-},y^- }(x,\theta^{})- \rho\right) \right]
\end{align*}
where the third line follows from the definition of $\Lcal(f_\theta(x),y)$ and    $\ell_{\cont}(f_\theta(x)_{y}-f_\theta(x)_{y'= y})=\ell_{\cont}(0)=1$, the forth line follows from the definition of $\rho$ and the expectation $\EE_{y^- \sim\mu_{y^{-}}(x)}$, the fifth line follows from $f_\theta(x)_{y}=\theta^{(H+1)}_{y}h_{\theta^{(1:H)}}(x)$ and  $f_\theta(x)_{y^{-}}=\theta^{(H+1)}_{y^{-}}h_{\theta^{(1:H)}}(x)$, the sixth line follows from the definition of $\Lcal_{\cont}^{x^{+},x^{-},y^- }$. 
The last line follows from the convexity of $\ell_{\cont}$ and Jensen's inequality: i.e.,  
\begin{align*}
&\psi_{x,1,2}([u_h^+,u_h^-]^{\top})
\\&=\ell_{\cont}(\EE_{x^+ \sim \mu_{x^{+}}(x)}\EE_{x^-\sim\mu_{x^{-}}(x)}[(h_{\theta^{(1:H)}}(x^+ )-h_{\theta^{(1:H)}}(x^-))^{\top}h_{\theta^{(1:H)}}(x)])
\\ & \le\EE_{x^+ \sim \mu_{x^{+}}(x)}\EE_{x^-\sim\mu_{x^{-}}(x)}\ell_{\cont}((h_{\theta^{(1:H)}}(x^+ )-h_{\theta^{(1:H)}}(x^-))^{\top}h_{\theta^{(1:H)}}(x)).  
\end{align*}
Therefore,
\begin{align*} 
& \EE_{x,y}[\Lcal(f_\theta(x),y)]
\\ &\le 
\EE_{x,y}\left[\frac{1}{1-\rho_{y}(\mu_{y^{-}}(x))}\left(\Lcal_{\cont}^{x^{+},x^-}(x,\theta^{(1:H)})+\Lcal_{\cont}^{x^{+},x^{-},y^- }(x,\theta^{})-\rho_{y}(\mu_{y^{-}}(x))\right) \right]   
\\ & \le \EE_{x}\left[\frac{1}{1- \bar \rho(x)  }\left(\Lcal_{\cont}^{x^{+},x^-}(x,\theta^{(1:H)})+\Lcal_{\cont}^{x^{+},x^{-},y^- }(x,\theta^{})-\underline{\rho}(x)\right) \right]
\end{align*}
where we used $\underline{\rho}(x)  \le \rho_{y}(\mu_{y^{-}}(x))\le  \bar \rho(x) \in [0, 1)$.
\end{proof}

\begin{lemma} \label{lemma:gen_2}
Let $S \mapsto  f_{\theta(x,S)}(x)$ be an uniformly $\zeta$-stable tailoring algorithm. Then, for any $\delta>0$, with probability at least $1-\delta$ over  an i.i.d. draw of $n$ i.i.d. samples  $S=((x_i, y_i))_{i=1}^n$, the following holds:
\begin{align*}
\EE_{x,y}[\Lcal(f_{\theta(x,S)}(x),y)] \le \frac{1}{n}\sum_{i=1}^{n}\Lcal(f_{\theta(x_i,S)}(x_i),y_i)+\frac{\zeta}{n}+(2\zeta+c) \sqrt{\frac{\ln(1/\delta)}{2n}}. 
\end{align*}
\end{lemma}
\begin{proof}[Proof of Lemma \ref{lemma:gen_2}]
 Define $\varphi_1(S)=\EE_{x,y}[\Lcal(f_{\theta(x,S)}(x),y)] $ and $\varphi_2(S)=\frac{1}{n}\sum_{i=1}^{n}\Lcal(f_{\theta(x_i,S)}(x_i),y_i) $, and $
 \varphi(S)=\varphi_1(S) -\varphi_2(S) $.  To apply McDiarmid's inequality to $\varphi(S)$, we compute an upper bound on $|\varphi(S)-\varphi(S')|$ where  $S$ and $S'$  be two training datasets differing by exactly one point of an arbitrary index $i_{0}$; i.e.,  $S_i= S'_i$ for all $i\neq i_{0}$ and $S_{i_{0}} \neq S'_{i_{0}}$, where $S'=((x_i', y_i'))_{i=1}^n$. Let $\zetat = \frac{\zeta}{n}$
Then, 
$$
|\varphi(S)-\varphi(S')|\le |\varphi_1(S)-\varphi_1(S')  | + |\varphi_2(S)-\varphi_2(S')|.
$$
For the first term, using the $\zeta$-stability, 
\begin{align*}
|\varphi_1(S)-\varphi_1(S')  |  &\le\EE_{x,y}[|\Lcal(f_{\theta(x,S)}(x),y)-\Lcal(f_{\theta(x,S')}(x),y)|] \\ & \le \zetat. 
\end{align*}
For the second term, using $\zeta$-stability and the upper bound $c$ on per-sample loss, 
\begin{align*}
 |\varphi_2(S)-\varphi_2(S')| & \le \frac{1}{n} \sum_{i\neq i_0} |\Lcal(f_{\theta(x_i,S)}(x_i),y_i) -\Lcal(f_{\theta(x_i,S')}(x_i),y_i) | + \frac{c}{n} 
\\ & \le \frac{(n-1)\zetat}{n} + \frac{c}{n} \le \zetat+ \frac{c}{n}.  
\end{align*}
Therefore, $|\varphi(S)-\varphi(S')|\le 2\zetat + \frac{c}{n}$. By McDiarmid's inequality, for any $\delta>0$, with probability at least $1-\delta$,
$$
\varphi(S) \le  \EE_{S}[\varphi(S)] + (2 \zeta+c) \sqrt{\frac{\ln(1/\delta)}{2n}}.
$$
The reset of the proof bounds the first term $\EE_{S}[\varphi(S)]$. By the linearity of expectation,
\begin{align*}
\EE_{S}[\varphi(S)] =\EE_{S}[\varphi_1(S)]-\EE_{S}[\varphi_1(S)]. 
\end{align*}
For the first term,
$$
\EE_{S}[\varphi_1(S)]=\EE_{S,x,y}[\Lcal(f_{\theta(x,S)}(x),y)].
$$
For the second term,
using the linearity of expectation,
\begin{align*}
\EE_{S}[\varphi_2(S)] &=\EE_{S}\left[\frac{1}{n}\sum_{i=1}^{n}\Lcal(f_{\theta(x_i,S)}(x_i),y_i)\right]
\\ & = \frac{1}{n}\sum_{i=1}^{n}\EE_{S}[\Lcal(f_{\theta(x_i,S)}(x_i),y_i)]
\\ & = \frac{1}{n}\sum_{i=1}^{n}\EE_{S,x,y}[\Lcal(f_{\theta(x,S^i_{x,y})}(x),y)],
\end{align*}
where $S^i$ is a sample of $n$ points such that  $(S^i_{x,y})_j=S_{j}$ for $j \neq i$ and $(S^i_{x,y})_i=(x,y)$. By combining these,
using the linearity of expectation and  $\zeta$-stability,
\begin{align*}
\EE_{S}[\varphi(S)] &=\frac{1}{n}\sum_{i=1}^n \EE_{S,x,y}[\Lcal(f_{\theta(x,S)}(x),y)-\Lcal(f_{\theta(x,S^i_{x,y})}(x),y)]
\\ & \le\frac{1}{n}\sum_{i=1}^n \EE_{S,x,y}[|\Lcal(f_{\theta(x,S)}(x),y)-\Lcal(f_{\theta(x,S^i_{x,y})}(x),y)|]
\\ & \le \frac{1}{n} \sum_{i=1}^n \zetat =\zetat.  
\end{align*}
Therefore, $\EE_{S}[\varphi(S)]\le \zetat$.

\end{proof}

\begin{proof}[Proof of Theorem \ref{thm:cont_2}]
For any $\theta$ and $\kappa \in [0,1]$, 
$$
\EE_{x,y}[\Lcal(f_\theta(x),y)] = \kappa \EE_{x,y}[\Lcal(f_\theta(x),y)] + (1-\kappa) \EE_{x,y}[\Lcal(f_\theta(x),y)]. 
$$
Applying Lemma \ref{lemma:contrastive} for the first term and Lemma \ref{lemma:gen_2} yields the desired statement.

\end{proof}

\subsection{Statement and proof of Theorem \ref{thm:cont_1}}

\begin{theorem} \label{thm:cont_1}
Let $\Fcal$ be an arbitrary set of maps $x\mapsto f_{\theta_{x}}(x)$. Then, for any $\delta>0$, with probability at least $1-\delta$ over  an i.i.d. draw of $n$ i.i.d. samples  $((x_i, y_i))_{i=1}^n$, the following holds: for all maps $(x\mapsto f_{\theta_{x}}(x))\in\Fcal$ and any $\kappa \in [0,1]$, we have that
$
\EE_{x,y}[\suploss(f_{\theta_x}(x),y)]\le \kappa \EE_{x}\left[\Lcal_{\cont}^{}(x,\theta^{}_{x}) \right]+ (1-\kappa)\mathcal{J}'$, where
$
\mathcal{J}' =\frac{1}{n}\sum_{i=1}^{n}\suploss(f_{\theta_{x_i}}(x_{i}),y_i)+2\Rcal_{n}(\suploss \circ \Fcal)+c \sqrt{(\ln(1/\delta))/(2n)}. 
$
\end{theorem}

The following lemma is used along with  Lemma \ref{lemma:contrastive}
to prove the statement of this theorem.

\begin{lemma} \label{lemma:gen_1}
Let $\Fcal$ be an arbitrary set of maps $x\mapsto f_{\thetax}(x)$. For any $\delta>0$, with probability at least $1-\delta$ over  an i.i.d. draw of $n$ i.i.d. samples  $((x_i, y_i))_{i=1}^n$, the following holds: for all maps $(x\mapsto f_{\thetax}(x) )\in\Fcal$,
\begin{align*}
\EE_{x,y}[\Lcal(f_{\thetax}(x),y)]\le \frac{1}{n}\sum_{i=1}^{n}\Lcal(f_\thetaxi(x_i),y_i)+2\Rcal_{n}(\Lcal \circ \Fcal)+c \sqrt{\frac{\ln(1/\delta)}{2n}}.   
\end{align*}
\end{lemma}
\begin{proof}[Proof of Lemma \ref{lemma:gen_1}]
 Let $S=((x_i, y_i))_{i=1}^n$ and $S'=((x_i', y_i'))_{i=1}^n$. Define 
 $$
 \varphi(S)= \sup_{(x\mapsto f_{\thetax}(x) ) \in \Fcal} \EE_{x,y}[\Lcal(f_{\thetax}(x),y)]-\frac{1}{n}\sum_{i=1}^{n}\Lcal(f_\thetaxi(x_i),y_i).
 $$ To apply McDiarmid's inequality to $\varphi(S)$, we compute an upper bound on $|\varphi(S)-\varphi(S')|$ where  $S$ and $S'$ be two training datasets differing by exactly one point of an arbitrary index $i_{0}$; i.e.,  $S_i= S'_i$ for all $i\neq i_{0}$ and $S_{i_{0}} \neq S'_{i_{0}}$. Then,
$$
\varphi(S')-\varphi(S) \le\sup_{(x\mapsto f_{\thetax}(x) ) \in \Fcal}\frac{\Lcal(f_{\theta(x_{i_0})}(x_{i_0}),y_{i_0})-\Lcal(f_{\theta(x'_{i_0})}(x'_{i_0}),y'_{i_0})}{n} \le \frac{c}{n}.
$$
Similarly, $\varphi(S)-\varphi(S') \le \frac{c}{n}$. Thus, by McDiarmid's inequality, for any $\delta>0$, with probability at least $1-\delta$,
$$
\varphi(S) \le  \EE_{S}[\varphi(S)] + c \sqrt{\frac{\ln(1/\delta)}{2n}}.
$$
Moreover, with $f(x)=f_{\thetax}(x)$,
\begin{align*}
\EE_{S}[\varphi(S)] 
 &  = \EE_{S}\left[\sup_{f \in \Fcal} \EE_{S'}\left[\frac{1}{n}\sum_{i=1}^{n}\Lcal(f_{\theta(x_i')}(x_i'),y_i')\right]-\frac{1}{n}\sum_{i=1}^{n}\Lcal(f_\thetaxi(x_i),y_i)\right]   
 \\ &  \le\EE_{S,S'}\left[\sup_{f \in \Fcal} \frac{1}{n}\sum_{i=1}^n (\Lcal(f_{\theta(x'_i)}(x'_i),y'_i)-\Lcal(f_\thetaxi(x_i),y_i)\right] 
 \\ & \le \EE_{\xi, S, S'}\left[\sup_{f\in \Fcal} \frac{1}{n}\sum_{i=1}^n  \xi_i(\Lcal(f_{\theta(x'_{i})}(x'_{i}),y'_{i})-\Lcal(f_\thetaxi(x_i),y_i))\right]
 \\ &  \le2\EE_{\xi, S}\left[\sup_{f\in \Fcal} \frac{1}{n}\sum_{i=1}^n  \xi_i\Lcal(f_\thetaxi(x_i),y_i))\right] 
\end{align*}
where  the fist line follows the definitions of each term, the second line uses the Jensen's inequality and the convexity of  the 
supremum, and the third line follows that for each $\xi_i \in \{-1,+1\}$, the distribution of each term $\xi_i (\Lcal(f_{\theta(x'_i)}(x'_i),y'_i)-\Lcal(f_\thetaxi(x_i),y_i))$ is the  distribution of  $(\Lcal(f_{\theta(x'_i)}(x'_i),y'_i)-\Lcal(f_\thetaxi(x_i),y_i))$  since $\bar S$ and $\bar S'$ are drawn iid with the same distribution. The forth line uses the subadditivity of supremum. 
\end{proof}

\begin{proof}[Proof of Theorem \ref{thm:cont_1}]
For any $\theta$ and $\kappa \in [0, 1]$, 
$$
\EE_{x,y}[\Lcal(f_\theta(x),y)] = \kappa \EE_{x,y}[\Lcal(f_\theta(x),y)] + (1-\kappa) \EE_{x,y}[\Lcal(f_\theta(x),y)]. 
$$
Applying Lemma \ref{lemma:contrastive} for the first term and Lemma \ref{lemma:gen_1} yields the desired statement.  

\end{proof}

\subsection{Proof of Theorem \ref{thm:cont_3}}

\begin{proof}[Proof of Theorem \ref{thm:cont_3}]
Let $\theta^{(1:H)}$ be fixed. We first prove the first statement for the convexity. The function $\theta^{(H+1)} \mapsto\psi_{x, y,y^{-}}(\theta^{(H+1)})$ is convex, since it is a composition of a convex function $\ell_{\cont}$ and a affine function $z\mapsto(z_{y}-z_{y^{-}})h_{\theta^{(1:H)}}(x)) $. The function $\theta^{(H+1)} \mapsto \EE_{y^- \sim\mu_{y^{-}}(x)}[\psi_{x, y,y^{-}}(\theta^{(H+1)})-\psi_{x,1,2}([u_h^+,u_h^-]^{\top})]$ is convex since the expectation and affine translation preserves the convexity. Finally, $\theta^{(H+1)} \mapsto\Lcal_{\cont}^{x^{+},x^{-},y^- }(x,\theta^{(1:H)},\theta^{(H+1)})$ is convex since it is the piecewise maximum of the convex functions 
$$
\theta^{(H+1)} \mapsto \EE_{y^- \sim\mu_{y^{-}}(x)}[\psi_{x, y,y^{-}}(\theta^{(H+1)})-\psi_{x,1,2}([u_h^+,u_h^-]^{\top})]
$$ for each $y$.

We now prove the second statement of the theorem for the inequality. Let us write $ \mu_{x^{+}}=  \mu_{x|y}$ and $\mu_{x^{-}}=  \mu_{x|y^-}$. Let $U=[u_1,u_2]^{\top} \in \RR^{m_y \times( m_H+1)}$ where $u_y=\EE_{x \sim  \mu_{x|y}}[h_{\theta^{(1:H)}}(x )]$ for $y\in \{1,2\}$. Then, 
$$u_h^+=\EE_{x^+ \sim \mu_{x^{+}}(x)}[h_{\theta^{(1:H)}}(x^+ )]=\EE_{x^+ \sim \mu_{x|y}}[h_{\theta^{(1:H)}}(x^+ )]=u_{y},$$ 
and 
$$u_h^-=\EE_{x^- \sim \mu_{x^{-}}(x)}[h_{\theta^{(1:H)}}(x^- )]=\EE_{x^- \sim \mu_{x|y^-}}[h_{\theta^{(1:H)}}(x^- )]=u_{y^{-}}.
$$
Therefore, 
$$
\psi_{x,1,2}([u_h^+,u_h^-]^{\top})=\psi_{x,y,y^{-}}(U),
$$
with which
$$
\Lcal_{\cont}^{x^{+},x^{-},y^- }(x, \theta)=\max_{y}\EE_{y^- \sim\mu_{y^{-}}(x)}[\psi_{x, y,y^{-}}(\theta^{(H+1)})-\psi_{x,y,y^{-}}(U)].
$$ 
Since $U$ and $\theta^{(1:H)}$ do not contain $\theta^{(H+1)}$, for any $U,\bar\theta^{(1:H)}$, there exists $\theta^{(H+1)}=U$ for which $\psi_{x, y,y^{-}}(\theta^{(H+1)})-\psi_{x,y,y^{-}}(U)=0$ and hence $\Lcal_{\cont}^{x^{+},x^{-},y^- }(x, \theta)=0$. Therefore,
\begin{align*}
& \inf_{\theta^{(H+1)}\in \RR^{m_y \times( m_H+1)}} \Lcal_{\cont}^{x^{+},x^{-}}(x,\theta^{(1:H)}) +\Lcal_{\cont}^{x^{+},x^{-},y^-}(x, \theta^{(1:H)},\theta^{(H+1)})
\\ & =\Lcal_{\cont}^{x^{+},x^{-}}(x,\theta^{(1:H)}) +\inf_{\theta^{(H+1)}\in \RR^{m_y \times( m_H+1)}} \Lcal_{\cont}^{x^{+},x^{-},y^-}(x, \theta^{(1:H)},\theta^{(H+1)})
\\ & \le \Lcal_{\cont}^{x^{+},x^{-}}(x,\theta^{(1:H)}) .
\end{align*}
\end{proof}

\subsection{Proof of Remark \ref{coro:gen_1}}
\begin{proof}[Proof of Remark \ref{coro:gen_1}]
For any $\theta$, 
$$
\EE_{x,y}[\Lcal(f_\theta(x),y)] = \inf_{\kappa \in [0,1]} \kappa \EE_{x,y}[\Lcal(f_\theta(x),y)] + (1-\kappa) \EE_{x,y}[\Lcal(f_\theta(x),y)].
$$
Applying  Lemma \ref{lemma:gen_2} for Theorem \ref{thm:cont_2} (and Lemma \ref{lemma:gen_1} for Theorem \ref{thm:cont_1}) to the second term and the assumption $\EE_{x,y}[\Lcal(f_{\theta}(x),y)] \le\EE_{x}[\Lcal_{\un}(f_{\theta}(x))]$ to the first term yields the desired statement. 

\end{proof}
\newpage
\section{Details and description of \method\ }\label{app:method}
In this section we describe~\method\ in greater detail: its implementation, different variants and run-time costs. Note that, although this section is written from the perspective of meta-tailoring,~\method\ is also applicable to meta-learning, we provide pseudo-code in algorithm~\ref{alg:meta-learning}. The main idea behind~\method\ is to optimize only conditional normalization (CN) parameters $\gamma^{(l)}, \beta^{(l)}$ in the inner loop and optimize all the other weights $w$ in the outer loop. 
To simplify notation for implementation, in this subsection only, we overload notations to make them work over a mini-batch as follows. Let $b$ be a (mini-)batch size. Given $X\in\RR^{b \times m_0}$, $\gamma\in\RR^{b \times \sum_l m_l}$ and $\beta\in\RR^{b \times \sum_l m_l}$, let $(f_{w,\gamma,\beta}(X))_i=f_{w,\gamma_i,\beta_i}(X_i)$ where $X_i$, $\gamma_i$, and $\beta_i$ are the transposes of the $i$-th row vectors of $X$, $\gamma$ and $\beta$, respectively. Similarly, $\suploss$ and $\tailorloss$ are used over a mini-batch. We also refer to $\theta=(w,\gamma,\beta)$.   

\paragraph{Initialization of $\gamma,\beta$} In the inner loop we always initialize $\gamma = \mathbf{1}_{b,\sum_l m_l}, \beta = \mathbf{0}_{b,\sum_l m_l}$. More complex methods where the initialization of these parameters is meta-trained are also possible. However, we note two things: 
\begin{enumerate}
    \item By initializing to the identity function, we can pick an architecture trained with regular inductive learning, add CN layers without changing predictions and perform tailoring. In this manner, the prediction algorithm is the same regardless of whether we trained with meta-tailoring or without the CN parameters.
    \item We can add a previous normalization layer with weights $\gamma'^{(l)},\beta'^{(l)}$ that are trained in the outer loop, having a similar effect than meta-learning an initialization. However, we do not do it in our experiments.
\end{enumerate}

\paragraph{First and second order versions of~\method:} $w$ affect $\suploss$ in two ways: first, they directly affect the evaluation $f_{w,\gamma_s,\beta_s}(X)$ by being weights of the neural network; second, they affect $\nabla_\beta\tailorloss,\nabla_\gamma\tailorloss$ which affects $\gamma_s,\beta_s$ which in turn affect $\suploss$. Similar to {\sc MAML}~\citep{finn2017model}, we can implement two versions: in the first order version we only take into account the first effect, while in the second order version we take into account both effects. The first order version has three advantages:
\begin{enumerate}
    \item It is very easy to code: the optimization of the inner parameters and the outer parameters are detached and we simply need to back-propagate $\tailorloss$ with respect to $\beta,\gamma$ and $\suploss$ with respect to $w$. This version is easier to implement than most meta-learning algorithms, since the parameters in the inner and outer loop are different.
    \item It is faster: because we do not back-propagate through the optimization, the overall computation graph is smaller. 
    \item It is more stable to train: second-order gradients can be a bit unstable to train; this required us to lower the inner tailoring learning rate in experiments of section~\ref{subsec:planets} for the second-order version.
\end{enumerate}
The second-order version has one big advantage: it optimizes the true objective, taking into account how $\tailorloss$ will affect the update of the network. This is critical to linking the unsupervised loss to best serve the supervised loss by performing informative updates to the CN parameters. 

\paragraph{WarpGrad-inspired stopping of gradients and subsequent reduction in memory cost:} WarpGrad~\citep{flennerhag2019meta}\ was an inspiration to~\method\, suggesting to interleave layers that are adapted in the inner loop with layers only adapted in the outer loop. In contrast to WarpGrad, we can evaluate inputs (in meta-tailoring) or tasks (in meta-learning) in parallel, which speeds up training and inference. This also simplifies the code because we do not have to manually perform batches of tasks by iterating through them. 

WarpGrad also proposes to stop the gradients between inner steps; we include this idea as an optional operation in~\method, as shown in line 12 of~\ref{alg:CNGRAD}. The advantage of adding it is that it decreases the memory cost when performing multiple inner steps, as we now only have to keep in memory the computation graph of the last step instead of all the steps, key when the networks are very deep like in the experiments of section~\ref{sec:adv}. Another advantage is that it makes training more stable, reducing variance, as back-propagating through the optimization is often very noisy for many steps. At the same time it adds bias, because it makes the greedy assumption that locally minimizing the decrease in outer loss at every step will lead to low overall loss after multiple steps.

\paragraph{Computational cost:} in~\method\ we perform multiple forward and backward passes, compared to a single forward pass in the usual setting. In particular, if we perform $s$ tailoring steps, we execute $(s+1)$ forward steps and $s$ backward steps, which usually take the same amount of time as the forward steps. Therefore, in its naive implementation, this method takes about $2s+1$ times more than executing the regular network without tailoring.

However, it is well-known that we can often only adapt the higher layers of a network, while keeping the lower layers constant. Moreover, our proof about the capacity of~\method\ to optimize a broad range of inner losses only required us to adapt the very last CN layer $\gamma^{(H)},\beta^{(H)}$. This implies we can put the CN layers only on the top layer(s). In the case of only having one CN layer at the last network layer, we only require one initial full forward pass (as we do without tailoring). Then, we have $s$ backward-forward steps that affect only the last layer, thus costing $\frac{1}{H}$ in case of layers of equivalent cost. This leads to a factor of $1+\frac{2s}{H}$ in cost, which for $s$ small and $H$ large (typical for deep networks), is a very small overcost. Moreover, for tailoring and meta-tailoring, we are likely to get the same performance with smaller networks, which may compensate the increase in cost.

\paragraph{Meta-learning version:}~\method\ can also be used in meta-learning, with the advantage of being provably expressive, very efficient in terms of parameters and compute, and being able to parallelize across tasks. We show the pseudo-code for few-shot supervised learning in algorithm~\ref{alg:meta-learning}. There are two changes to handle the meta-learning setting: first, in the inner loop, instead of the unsupervised tailoring loss we optimize a supervised loss on the training (support) set. Second, we want to share the same inner parameters $\gamma,\beta$ for different samples of the same task. To do so we add the operation "repeat\_interlave" (PyTorch notation), which makes $k$ contiguous copies of each parameter $\gamma,\beta$, before feeding them to the network evaluation. In doing so, gradients coming from different samples of the same task get pooled together. At test time we do the same for the $k'$ queries ($k'$ can be different than $k$). Note that, in practice, this pooling is also used in meta-tailoring when we have more than one data augmentation within $\tailorloss$.
\begin{algorithm}[t]
\SetAlgoLined
\DontPrintSemicolon
\begin{flushleft}
  \SetKwFunction{algo}{algo}\SetKwFunction{proc}{proc}
  \SetKwProg{myalg}{Algorithm}{}{}
  \SetKwProg{myproc}{Subroutine}{}{}
  \myproc{Training($f$, $\suploss$, $\lambda_{sup}$, $\tailorloss$, $\lambda_{tailor}$, $steps$,$\supdata$)}{
  randomly initialize $w$ \tcp*[r]{All parameters except $\gamma,\beta$; trained in outer loop}
  \While{not done}{
    \For(\tcp*[f]{$b$ batch size}){$0\leq i\leq n/b$ }{
        $X,Y = x_{ib:i(b+1)}, y_{ib:i(b+1)}$\;
        $\gamma_0 = \mathbf{1}_{b,\sum_l m_l}$\;
        $\beta_0 = \mathbf{0}_{b,\sum_l m_l}$\;
        
        \For{$1\leq s \leq steps$}{
            $\gamma_s = \gamma_{s-1} - \lambda_{tailor}\nabla_{\gamma}\tailorloss(w,\gamma_{s-1},\beta_{s-1},X)$\;\tcp*[r]{Inner step w.r.t. $\gamma$}

            $\beta_s = \beta_{s-1} - \lambda_{tailor}\nabla_{\beta}\tailorloss(w,\gamma_{s-1},\beta_{s-1},X)$\;\tcp*[r]{Inner step w.r.t. $\beta$}

        $\gamma_s,\beta_s = \gamma_s.detach(), \beta_s.detach()\;$ \tcp*[r]{Optional operation, only in $1^{st}$ order CNGrad: WarpGrad detach to avoid back-proping through multiple steps; reducing memory, and increasing stability, but adding bias.}
        $w = w - \lambda_{sup}\nabla_w \suploss\left(f_{w,\gamma_s,\beta_s}(X), Y)\right)$\; \tcp*[r]{Outer step}
        }
        
    }
  }
  \Return $w$\;
  }

  \myproc(\tcp*[f]{For meta-tailoring \& tailoring}){Prediction($f$, $w$, $\tailorloss$, $\lambda$, $steps$, $X$)}{
    \tcp*[r]{X contains multiple inputs, with independent tailoring processes}
    $b = X.shape[0]$ \tcp*[r]{number of inputs}
    $\gamma_0 = \mathbf{1}_{b,\sum_l m_l}$\;
    $\beta_0 = \mathbf{0}_{b,\sum_l m_l}$\;
    
    \For{$1\leq s \leq steps$}{
        $\gamma_s = \gamma_{s-1} - \lambda\nabla_{\gamma}\tailorloss(w,\gamma_{s-1},\beta_{s-1},X)$\;

        $\beta_s = \beta_{s-1} - \lambda\nabla_{\beta}\tailorloss(w,\gamma_{s-1},\beta_{s-1},X)$\;
    }
  \Return $f_{w,\gamma_{steps},\beta_{steps}}(X)$\;
  }
  \end{flushleft}
\caption{\method \label{alg:CNGRAD} for meta-tailoring}
\end{algorithm}

\begin{algorithm}[t]
\SetAlgoLined
\begin{flushleft}
\DontPrintSemicolon
  \SetKwFunction{algo}{algo}\SetKwFunction{proc}{proc}
  \SetKwProg{myalg}{Algorithm}{}{}
  \SetKwProg{myproc}{Subroutine}{}{}
  \myproc{Meta-training($f$, $\suploss$, $\lambda_{inner}$, $\lambda_{outer}$, $steps$,$\mathcal{T}$)}{
  randomly initialize $w$ \tcp*[r]{All parameters except $\gamma,\beta$; trained in outer loop}
  \While{not done}{
    \For(\tcp*[f]{$b$ batch size}){$0\leq i\leq n/b$ }{
        $X_{train}, Y_{train} = [\ ], [\ ]$\;
        $X_{test}, Y_{test} = [\ ], [\ ]$\;
        \For {$ib\leq j\leq i(b+1)$}{
            $(inp, out) \sim_k \mathcal{T}_j$\tcp*[r]{Take $k$ samples from each task for training}
            $X.append\left(inp\right) ; Y.append\left(out\right)$\;
            $(query, target) \sim_k' \mathcal{T}_j$\tcp*[r]{Take $k'$ samples from each task for testing}
            $X.append\left(query\right) ; Y.append\left(target\right)$\;
        }
        \tcp*[f]{We can now batch evaluations of multiple tasks}
        $X_{train}, Y_{train} = concat\left(X_{train}, dim=0\right), concat\left(Y_{train}, dim=0\right)$\;
        $X_{test}, Y_{test} = concat\left(X_{test}, dim=0\right),  concat\left(Y_{test}, dim=0\right)$\;
        $\gamma_0 = \mathbf{1}_{b,\sum_l m_l}$\;
        $\beta_0 = \mathbf{0}_{b,\sum_l m_l}$\;
        \For{$1\leq s \leq steps$}{
             \tcp*[f]{We now repeat the CN parameters $k$ times so that samples from the same task share the same CN parameters}\;
            $\gamma^{tr}_{s-1}, \beta^{tr}_{s-1} = \gamma_{s-1}.repeat\_interleave(k,1),  \beta_{s-1}.repeat\_interleave(k,1)$\;
            $\gamma_s = \gamma_{s-1} - \lambda_{innner}\nabla_{\gamma}\suploss(f_{w,\gamma^{tr}_{s-1},\beta^{tr}_{s-1}}(X_{train}), Y_{train})$\;
            $\beta_s = \beta_{s-1} - \lambda_{innner}\nabla_{\beta}\suploss(f_{w,\gamma^{tr}_{s-1},\beta^{tr}_{s-1}}(X_{train}), Y_{train})$\;
            $\gamma^{test}_{s}, \beta^{test}_{s} = \gamma_{s}.repeat\_interleave(k',1),  \beta_{s}.repeat\_interleave(k',1)$\;
            $w = w - \lambda_{outer}\nabla_w \suploss\left(f_{w,\gamma^{test}_s,\beta^{test}_s}(X_{test}), Y_{test})\right)$\;
            $\beta_s,\gamma_s = \beta_s.detach(), \gamma_s.detach()$\; \tcp*[f]{WarpGrad detach to not backprop through multiple steps}
        }
    }
  }
  \Return $w$\;
  }

  \myproc{Meta-test($f$, $w$, $\suploss$, $\lambda_{inner}$,$steps$,$X_{train}$, $Y_{train}$, $X_{test}$)}{
    \tcp*[r]{Assuming a single task, although we could evaluate multiple tasks in parallel as in meta-training.}
    $\gamma_0 = \mathbf{1}_{1,\sum_l m_l}$ \tcp*[r]{single $\gamma,\beta$ because we only have one task}
    $\beta_0 = \mathbf{0}_{1,\sum_l m_l}$\;
    \For{$1\leq s \leq steps$}{
            $\gamma^{tr}_{s-1}, \beta^{tr}_{s-1} = \gamma_{s-1}.repeat\_interleave(k,1),  \beta_{s-1}.repeat\_interleave(k,1)$\;
            $\gamma_s = \gamma_{s-1} - \lambda_{innner}\nabla_{\gamma}\suploss(f_{w,\gamma^{tr}_{s-1},\beta^{tr}_{s-1}}(X_{train}), Y_{train})$\;
            $\beta_s = \beta_{s-1} - \lambda_{innner}\nabla_{\beta}\suploss(f_{w,\gamma^{tr}_{s-1},\beta^{tr}_{s-1}}(X_{train}), Y_{train})$\;
    }
    $\gamma^{test}_{steps}, \beta^{test}_{steps} = \gamma_{steps}.repeat\_interleave(k',1),  \beta_{steps}.repeat\_interleave(k',1)$\;
  \Return $f_{w,\gamma^{test}_{steps},\beta^{test}_{steps}}(X_{test})$\;
  }
 \end{flushleft}
\caption{\method \label{alg:meta-learning} for meta-learning}
\end{algorithm}

\clearpage
\newpage
\section{Experimental details of physics experiments}\label{app:planets}
\paragraph{Dataset generation}
As mentioned in the main text, 5-body systems are chaotic and most random configurations are unstable. To generate our dataset we used Finite Differences to optimize 5-body dynamical systems that were stable for 200 steps (no planet collisions and no planet outside a predetermined grid) and then picked the first 100 steps of their trajectories, to ensure dynamical stability. To generate each trajectory, we randomly initialized 5 planets within a 2D grid of size $w=600,h=300$, with a uniform probability of being anywhere in the central grid of size $w/2,h/2$, each with a mass sampled from a uniform between $[0.15,0.25]$ (arbitrary units) and with random starting velocity initialized with a Gaussian distribution. We then use a 4th order Runge-Kutta integrator to accurately simulate the ODE of the dynamical system until we either reach 200 steps, two planets get within a certain critical distance from each other or a planet gets outside the pre-configured grid. If the trajectory reached 200 steps, we added it to the dataset; otherwise we made a small random perturbation to the initial configuration of the planets and tried again. If the new perturbation did not reach 200 steps, but lasted longer we kept the perturbation as the new origin for future initialization perturbations, otherwise we kept our current initialization. Once all the datasets were generated we picked those below a threshold mean mass and partitioned them randomly into train and test. Finally, we normalize each of the 25 dimensions (5 planets and for each planet $x$,$y$,$v_x$,$v_y$,$m$) to have mean zero and standard deviation one. For inputs, we use each state and as target we use the next state; therefore, each trajectory gives us 100 pairs.

For more details, we attach the code that generated the dataset.
\paragraph{Implementation of tailoring, meta-tailoring and~\method}

All of our code is implemented in PyTorch~\citep{paszke2019pytorch}, using the higher library~\citep{grefenstette2019generalized}(\url{https://github.com/facebookresearch/higher}) to implement the second-order version of~\method. We implemented a 3-layer feedforward neural network, with a conditional normalization layer after each layer except the final regression layer. The result of the network was added to the input, thus effectively predicting the delta between the current state and the next state. For both the first-order and second-order versions of~\method, we used the detachment of WarpGrad (line 12 in algorithm~\ref{alg:CNGRAD}). For more details, we also attach the implementation of the method.
\paragraph{Compute and hyper-parameter search}
To keep the evaluation as strict as possible, we searched all the hyper-parameters affecting the inductive baseline and our tailoring versions with the baseline and simply copied these values for tailoring and meta-tailoring. For the latter two, we also had to search for $\lambda_{tailor}$.

The number of epochs was 1000, selected with the inductive baseline, although more epochs did not substantially affect performance in either direction. We note that meta-tailoring performance plateaued earlier in terms of epochs, but we left it the same for consistency. Interestingly, we found that regularizing the physics loss (energy and momentum conservation) helped the inductive baseline, even though the training data already has 0 physics loss. We searched over $[10^{-4},3\cdot10^{-4},10^{-3},3\cdot 10^{-3},10^{-2}]$ for the weight assigned to the physics loss and chose $2\cdot 10^{-3}$ for best performance in the inductive baseline. To balance between energy and momentum losses we multiplied the momentum loss by $10$ to roughly balance their magnitudes before adding them into a single physics loss, this weighting was not searched. We copied these settings for meta-tailoring.

In terms of the neural network architecture, we chose a simple model with 3 hidden layers of the same size and tried $[128,256,512]$ on the inductive baseline, choosing $512$ and deciding not to go higher for compute reasons and because we were already able to get much lower training loss than test loss. We copied these settings for the meta-tailoring setup. We note that since there are approximately $O(m_h^2)$ weight parameters, yet only $O(m_h)$ affine parameters used for tailoring, adding tailoring and meta-tailoring increase parameters roughly by a fraction $O(1/m_h)$, or about $0.2\%$. Also in the inductive baseline, we tried adding Batch Normalization~\citep{ioffe2015batch}, since it didn't affect performance we decided not to add it.

We chose the tailoring step size parameter by trying $[10^{-5}, 10^{-4},10^{-3},10^{-2}]$, finding $10^{-3}$ worked well while requiring less steps than using a smaller step size. We therefore used this step for meta-tailoring as well, which worked well for first-order~\method, but not for second-order~\method, whose training diverged. We thus lowered the tailoring step size to $10^{-4}$ for the second-order version, which worked well. We also tried clipping the inner gradients to increase the stability of the second-order method; since gains on preliminary experiments were small, we decided to keep it out for simplicity.

For meta-tailoring we only tried $2$ and $5$ tailoring steps (we wanted more than one step to show the algorithm capability, but few tailoring steps to keep inference time competitive). Since they worked similarly well, we chose $2$ steps to have a faster model. For the second-order version we also used $2$ steps, which performed much better than the inductive baseline and tailoring, but worse than the first-order version reported in the main text(about $20\%$ improvement of the second-order version vs. $7\%$ of tailoring and $35\%$ improvement of the first-order version).

For the baseline of optimizing the output we tried a step size of $10^{-4},10^{-3},10^{-2},10^{-1}$. Except for a step size of $10^{-1}$, results optimized the physics loss and always achieved a very small improvement, without overfitting to the physics loss. We thus chose a big learning rate and high number of steps to report the biggest improvement, of $0.7\%$.

\paragraph{Runs, compute and statistical confidence:} we ran each method 2 times and averaged the results. Note that the baseline of optimizing the output and \textit{tailoring} start from the inductive learning baseline, as they are meant to be methods executed after regular inductive training. This is why both curves start at the same point in Figure~\ref{fig:planet-results}. For those methods, we report the standard deviation of the mean estimate of the \textit{improvement}, since they are executed on the same run. Note that the standard deviation of the runs would be higher, but less appropriate. For meta-tailoring, we do use the standard deviation of the mean estimate of both runs, since they are independent from the inductive baseline.

All experiments were performed on a GTX 2080 Ti with 4 CPU cores.
\section{Experimental details on real pendulum}\label{app:soft_pendulum}

We modify the real pendulum of Hamiltonian Neural Networks~\citep{greydanus2019hamiltonian}. In particular we pick the energy of the system and use it as a tailoring loss.~\citet{greydanus2019hamiltonian} train a vanilla MLP and show that its non-conservation of energy results in poor generalization from train to test for long-term predictions. With HNNs that automatically discover an energy function and encode hamiltonian dynamics into the network architecture, the system conserves this proxy energy even in long predictions, resulting in better generalization. We meta-tailor the vanilla MLP, with no change in its architecture, beyond adding CN layers to efficiently perform tailoring. We try different inner learning rates ($1e-3,1e-2,1e-1,1e0$) as well as number of steps ($1,2,3$) and evaluate on long-term \textit{training} loss. Since training is 4 times as long as test, we divide training into 4 equally-big trajectories and choose the configuration with the best loss: $3$ steps and $1e-1$ inner learning rate. It is worth noting that these long term predictions use scipy's ODE integrator, which is also used for the vanilla MLP as well as for HNNs. We see that, by not fully enforcing energy conservation, meta-tailoring improves over both an inductive baseline of the same architecture and HNNs.

Experiments were performed with a Volta V-100 and 10 CPU cores, taking a couple of hours to run in total.

\section{Experimental details on contrastive learning}\label{app:contrastive}
We take the implementation of SimCLR~\cite{chen2020simple} from \url{https://github.com/leftthomas/SimCLR} evaluating on CIFAR-10~\cite{krizhevsky2009learning}. 

As detailed in the main text we train the vanilla SimCLR to get an unsupervised representation. We than train \textit{only} the linear layer with different amounts of training data, from 50 to 5000 points per class. Vanilla SimCLR follows regular inductive learning with supervised labels for the linear layer. Meta-tailoring uses the same augmentations provided by SimCLR and minimizes the SimCLR loss on each particular input before feeding the tailored representations to the linear layer. The linear layer is trained to take these adapted representations. We tried different hyper-parameters: $[4,8,16]$ augmentations, $1,2,$ inner optimization steps, inner learning rate of $[1e-1,3e-1,1e0,3e0,1e1,3e1]$ and whether to tailor the CN layers of the CNN representation or tailor the representations $h$ directly. We found very consistent results where all stable inner optimizations improved over vanilla SimCLR, and longer optimizations with larger learning rates and more augmentations gave bigger improvements. Tailoring the CN layers or the representations directly didn't make a big effect, the latter being slightly more stable, and providing somewhat larger results. It is also much faster as we do not need to back-prop back through the CNN. We thus chose $3$ steps $1e1$ learning rate, 16 augmentations and tailoring the representations directly. For TTT we kept the $3$ steps and tried $0,1e-4,1e-3,1e-2,1e-1,1e0,1e1$ learning rates. We noticed that all these inner optimizations were stable, yet performance degraded with learning rate. Thus the best learning rate was $0$, equivalent to not doing TTT.

For all methods we follow the code-base and keep the best validation. Keeping running averages or choosing concrete epochs gave very similar results because of the stability of training a linear layer. We averaged over 5 different trainings of the linear layer, all using the same SimCLR base. The TTT baseline uses the best SimCLR epoch for each of these 5 runs, so that using TTT with lr=0 gives exactly the same results as the baseline.

For TTT we trained on the rotation prediction task proposed in the original paper. To minimize differences, we use the same architecture as the MLP from SimCLR, except with 4 output logits, one per rotation. It achieves $80.5\%$ test accuracy on rotation prediction. TTT proposes to back-propagate this rotation prediction loss back at test-time, but does not take this procedure into account at training time. We consistently find that TTT worsens the performance, with the gap becoming worse as we increase the learning rate. We think the reason why this loss is helpful in a very similar dataset and architecture in~\citet{sun2019test}, yet hurts performance in this case is due to two factors. First, it can be observed in the original paper that rotation-prediction provides consistent, but small gains, in the 1-sample case, with much larger gains in its online multi-sample version. Second,~\citet{sun2019test} focus on out-of-distribution generalization, where weights are trained on a different data distribution and are thus sub-optimal. The linear layer receiving OOD inputs in this same-distribution case hurts performance, but in their OOD application, even the unmodified inputs were already OOD. Meta-tailoring takes the adaptation into account, thus the inputs of the linear layer are always in-distribution, thus being able to help performance. 

\begin{figure}
    \centering
    \includegraphics[width=0.75\linewidth]{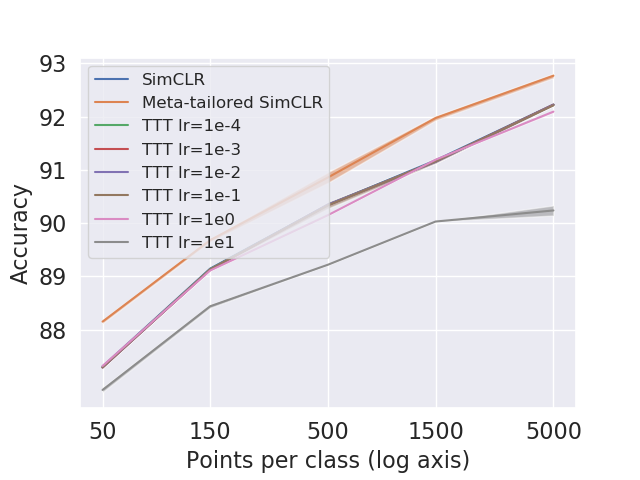}
    \caption{Test-time training (TTT) with its original rotation-prediction auxiliary task performs worse than vanilla SimCLR. Performance degrades as we increase the inner learning rate, thus increasing its power. }
    \label{fig:TTT_contrastive}
\end{figure}

Experiments were performed with a set of Volta V-100 and 10 CPU cores. SimCLR takes around a day to train. All other experiments training the linear layer from different initializations and for all data quantities take a few hours for a single set of hyper-parameters.
\section{Toy adversarial examples experiment}
We illustrate the tailoring process with a simple illuminating example using the data from~\citet{ilyas2019adversarial}.  They use discriminant analysis on the data in Figure~\ref{fig:toy_adversarial} and obtain the purple linear separator.  It has the property that, under assumptions about Gaussian distribution of the data, points above the line are more likely to have come from the blue class, and those below, from the red class.
This separator is not very adversarially robust, in the sense that, for many points, a perturbation with a small $\delta$ would change the assigned class. We improve the robustness of this classifier by tailoring it using the loss  

$\tailorloss(x,\theta) = \textrm{KL}(\phi(f_\theta(x))\; || \; \phi(f_\theta(x+\argmax_{|\delta|<\varepsilon} \sum_j e^{f_\theta(x+\delta)_j})))$, where $\textrm{KL}$ represents the KL divergence, $\phi$ is the logistic function, and $\phi(f_\theta(x)_i)$ is the probability of $x$ being in class $i$, so that $\phi(f_\theta(x))$ represents the entire class distribution.


With this loss, we can adjust our parameters $\theta$ so that the KL divergence between our prediction at $x$ is closer to the prediction at perturbed point $x+\delta$, over all perturbations in radius $\varepsilon$. Note that we initialized the models with the weights of~\citet{cohen2019certified} to speed up training in all ImageNet experiments and to avoid training divergence for CIFAR-10 with $\sigma=1$ (this divergence was already noted by~\citet{Zhai2020MACER}).
Each of the curves in Figure 3 corresponds to a decision boundary induced by tailoring the original separator with a different value for the maximum perturbation $\varepsilon$. Note that the resulting separators are non-linear, even though we are tailoring a linear separator, because the tailoring is local to the prediction point.  We also have the advantage of being able to choose different values of $\varepsilon$ at prediction time.

\paragraph{Hyper-parameters} the model does not have any hyper-parameters, as we use the model from~\citet{ilyas2019adversarial}, which is based on the mean $\mu$ and standard deviation $\sigma$ of the Gaussians. For tailoring, we used a $5\times5$ grid to initialize the inner optimization to find the point of highest probability within the $\epsilon$-ball. Using a single starting point did not work as gradient descent found a local optima. Using more ($10\times10$) did not improve results further, while increasing compute. We also experimented between doing a weighted average of the predictions by their energy to compute the tailoring loss or picking the element with the biggest energy. Results did not seem to differ much (likely because likelihood distributions are very peaked), so we picked the simplest option of imitating the element of highest probability. Doing a single tailoring step already worked well (we tried step sizes of $10^{-1},1,10,30$ with $10$ working best), so we kept that for simplicity and faster predictions.

Regarding compute, this experiment can be generated in a few minutes using a single GTX 2080 Ti.
\begin{figure}[h]
\begin{center}
    \includegraphics[width=0.5\textwidth]{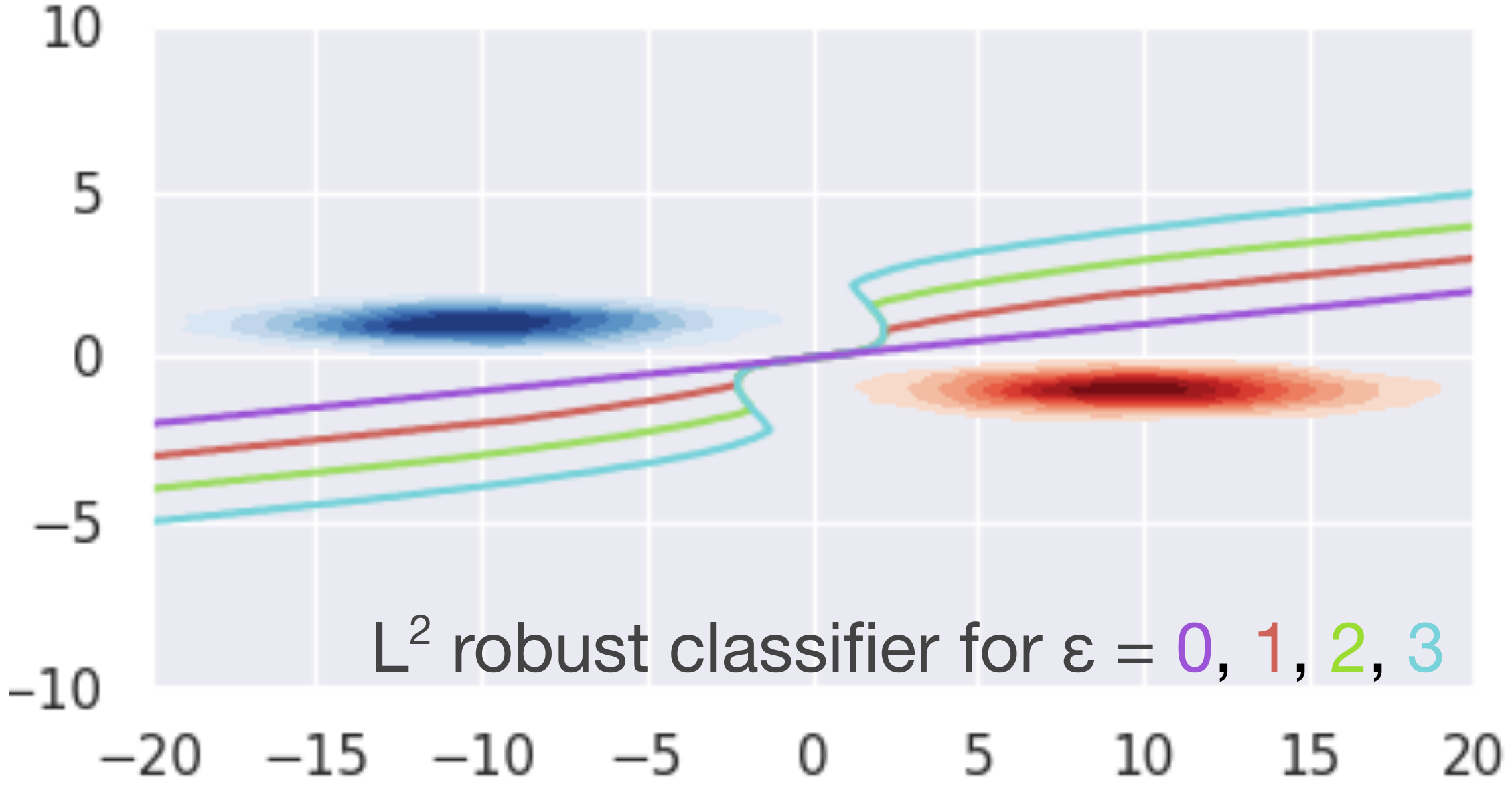}
\end{center}
  \caption{
  Decision boundary of our model at multiple levels of robustness on an example from~\citet{ilyas2019adversarial}.
  \label{fig:toy_adversarial}}
\end{figure}
\section{Experimental details of adversarial experiments}\label{app:adv}
Results for CIFAR-10 and ImageNet experiments comparing to state-of-the-art methods can be found in tables~\ref{table:cifar}, ~\ref{table:cifar-soa},  and~\ref{table:imagenet-soa}. Table~\ref{table:cifar} only includes results for Randomized Smoothing (RS). 

\begin{figure*}
\small
\begin{tabular}{c|l|rrrrrrrrrr|r}
    \toprule[1.5pt]
      $\sigma$ & \textbf{Method}& 0.0 & 0.25 & 0.5& 0.75 & 1.0&1.25 & 1.5 & 1.75 & 2.00 & 2.25 & ACR\\
    \midrule[2pt]
    \multirow{2}{*}{0.25} &  (Inductive) RS& 0.75 & 0.60 & 0.43 & 0.26 & 0.00 & 0.00 & 0.00 & 0.00 & 0.00 & 0.00 & 0.416\\
     & Meta-tailor RS& \textbf{0.80} & \textbf{0.66} & \textbf{0.48} & 0.29 & 0.00 & 0.00 & 0.00 & 0.00 & 0.00 & 0.00 & \textbf{0.452}\\
    \midrule[1pt]
    \multirow{2}{*}{0.50}  &  (Inductive) RS& 0.65 & 0.54 & 0.41 & 0.32 & 0.23 & 0.15 & 0.09 & 0.04 & 0.00 & 0.00 & 0.491\\
     & Meta-tailor RS& 0.68 & 0.57 & 0.45 & \textbf{0.33} & 0.23 & 0.15 & 0.08 & 0.04 & 0.00 & 0.00 & \textbf{0.542}\\
    \midrule[1pt]
    \multirow{2}{*}{1.00} &  (Inductive) RS& 0.47 & 0.39 & 0.34 & 0.28 & 0.21 & 0.17 & \textbf{0.14} & 0.08 & 0.05 & 0.03 & 0.458\\
    & Meta-tailor RS& 0.50 & 0.43 & 0.36 & 0.30 & \textbf{0.24} & \textbf{0.19} & \textbf{0.14} & \textbf{0.10} & \textbf{0.07} & \textbf{0.05} & \textbf{0.546}\\
    \bottomrule[2pt]
  
  \end{tabular}
  \caption{Percentage of points with certificate above different radii, and average certified radius (ACR) for on the CIFAR-10 dataset. Meta-tailoring improves the Average Certification Radius by $8.6\%,10.4\%,19.2\%$ respectively. Results for~\citet{cohen2019certified} are taken from~\citet{Zhai2020MACER} because they add more measures than the original work, with similar results. 
  \label{table:cifar}
  }
\end{figure*}

\begin{figure}
\small
\begin{tabular}{c|l|rrrrrrrrrr|r}
    \toprule[1.5pt]
      $\sigma$ & \textbf{Method}& 0.0 & 0.25 & 0.5& 0.75 & 1.0&1.25 & 1.5 & 1.75 & 2.00 & 2.25 & ACR\\
    \midrule[2pt]
    \multirow{4}{*}{0.25} & RandSmooth & 0.75 & 0.60 & 0.43 & 0.26 & 0.00 & 0.00 & 0.00 & 0.00 & 0.00 & 0.00 & 0.416\\
     & Salman & 0.74 & 0.67 & 0.57 & 0.47 & 0.00 & 0.00 & 0.00 & 0.00 & 0.00 & 0.00 & 0.538\\
     & MACER & \textbf{0.81} & \textbf{0.71} & \textbf{0.59} & 0.43 & 0.00 & 0.00 & 0.00 & 0.00 & 0.00 & 0.00 & \textbf{0.556}\\
     & Meta-tailored & 0.80 & 0.66 & 0.48 & 0.29 & 0.00 & 0.00 & 0.00 & 0.00 & 0.00 & 0.00 & 0.452\\
    \midrule[1pt]
    \multirow{4}{*}{0.50}  & RandSmooth& 0.65 & 0.54 & 0.41 & 0.32 & 0.23 & 0.15 & 0.09 & 0.04 & 0.00 & 0.00 & 0.491\\& Salman & 0.50 & 0.46 & 0.44 & 0.40 & 0.38 & 0.33 & 0.29 & 0.23 & 0.00 & 0.00 & 0.709\\
    & MACER & 0.66 & 0.60 & 0.53 & \textbf{0.46} & 0.38 & 0.29 & 0.19 & 0.12 & 0.00 & 0.00 & \textbf{0.726}\\
     & Meta-tailored & 0.68 & 0.57 & 0.45 & 0.33 & 0.23 & 0.15 & 0.08 & 0.04 & 0.00 & 0.00 & 0.542\\
    \midrule[1pt]
    \multirow{4}{*}{1.00} & RandSmooth & 0.47 & 0.39 & 0.34 & 0.28 & 0.21 & 0.17 & 0.14 & 0.08 & 0.05 & 0.03 & 0.458\\
    & Salman &  0.45 & 0.41 & 0.38 & 0.35 & \textbf{0.32} & 0.28 & \textbf{0.25} & \textbf{0.22} & \textbf{0.19} & \textbf{0.17} & 0.787\\
    & MACER & 0.45 & 0.41 & 0.38 & 0.35 & \textbf{0.32} & \textbf{0.29} & \textbf{0.25} & \textbf{0.22} &0.18 & 0.16 & \textbf{0.792}\\
    & Meta-tailored & 0.50 & 0.43 & 0.36 & 0.30 & 0.24 & 0.19 & 0.14 & 0.10 & 0.07 & 0.05 & 0.546\\
    \bottomrule[2pt]
  
  \end{tabular}
  \caption{Percentage of points with certificate above different radii, and average certified radius (ACR) for on the CIFAR-10 dataset, comparing with SOA methods. In contrast to pretty competitive results in ImageNet, meta-tailoring improves randomized smoothing, but not enough to reach SOA. It is worth noting that the SOA algorithms could also likely be improved via meta-tailoring.
  \label{table:cifar-soa}
  }
\end{figure}

\begin{figure}
\small
\begin{tabular}{c|l|rrrrrrr|r}
    \toprule[1.5pt]
      $\sigma$ & \textbf{Method}& 0.0 & 0.5& 1.0 & 1.5  & 2.0 & 2.5 & 3.0 & ACR\\
    \midrule[2pt]
    \multirow{4}{*}{0.25} & RandSmooth & 0.67 & 0.49 & 0.00 & 0.00 & 0.00 & 0.00 & 0.00 & 0.470\\
     & Salman & 0.65 & 0.56 & 0.00 & 0.00 & 0.00 & 0.00 & 0.00 & 0.528\\
     & MACER & 0.68 & \textbf{0.57} & 0.00 & 0.00 & 0.00 & 0.00 & 0.00 & \textbf{0.544}\\
     & Meta-tailored RS& \textbf{0.72} & 0.55 & 0.00 & 0.00 & 0.00 & 0.00 & 0.00 & 0.494\\
    \midrule[1pt]
    \multirow{4}{*}{0.50}  & RandSmooth& 0.57 & 0.46 & 0.37 & 0.29 & 0.00 & 0.00 & 0.00 & 0.720\\
     & Salman & 0.54 & 0.49 & \textbf{0.43} & \textbf{0.37} & 0.00 & 0.00 & 0.00 & 0.815\\
     & MACER & 0.64 & 0.53 & \textbf{0.43} & 0.31 & 0.00 & 0.00 & 0.00 & \textbf{0.831}\\
     & Meta-tailored RS& 0.66 & 0.54 & 0.42 & 0.31 & 0.00 & 0.00 & 0.00 & 0.819\\
    \midrule[1pt]
    \multirow{4}{*}{1.00} & RandSmooth & 0.44 & 0.38 & 0.33 & 0.26 & 0.19 & 0.15 & 0.12 & 0.863\\
    & Salman & 0.40 & 0.38 & 0.33 & 0.30 & \textbf{0.27} & \textbf{0.25} & \textbf{0.20} & 1.003\\
    & MACER & 0.48 & 0.37 & 0.34 & 0.30& 0.25 & 0.18 & 0.14 & 1.008\\
    & Meta-tailored RS& 0.52 & 0.45 & 0.36 & 0.31 & 0.24 & 0.20 & 0.15 & \textbf{1.032}\\
    \bottomrule[2pt]
  
  \end{tabular}
  \caption{Percentage of points with certificate above different radii, and average certified radius (ACR) for on the ImageNet dataset, including other SOA methods. Randomized smoothing with meta-tailoring are very competitive with other SOA methods, including having the biggest ACR for $\sigma=1$.
  \label{table:imagenet-soa}
  }
\end{figure}

\paragraph{Hyper-parameters and other details of the experiments} there are just three hyper-parameters to tweak for these experiments, as we try to remain as close as possible to the experiments from~\citet{cohen2019certified}. In particular, we tried different added noises $\nu\in[0.05,0.1,0.2]$ and tailoring inner steps $\lambda\in[10^{-3},10^{-2},10^{-1},10^0]$ for $\sigma=0.5$. To minimize compute, we tried these settings by tailoring (not meta-tailoring) the original model and seeing its effects on the smoothness and stability of optimization, choosing $\nu=0.1,\lambda=0.1$ (the fact that they're the same is a coincidence). We chose to only do a single tailoring step to reduce the computational burden, since robustness certification is very expensive, as each example requires 100k evaluations (see below). For simplicity and to avoid excessive tuning, we chose the hyper-parameters for $\sigma=0.5$ and copied them for $\sigma=0.25$ and $\sigma=1$. As mentioned in the main text, $\sigma=1$ required initializing our model with that of~\citet{cohen2019certified} (training wasn't stable otherwise), which is easy to do using \method.

In terms of implementation, we use the codebase of~\citet{cohen2019certified}(\url{https://github.com/locuslab/smoothing}) extensively, modifying it only in a few places, most notably in the architecture to include tailoring in its forward method. It is also worth noting that we had to deactivate their disabling of gradients during certification, because tailoring requires gradients. We chose to use the first-order version of~\method\, which made it much easier to keep our implementation very close to the original. It is likely that doing more tailoring steps would result in better performance.

We note that other works focused on adversarial examples, such as~\citet{Zhai2020MACER,salman2019provably}, improve on~\citet{cohen2019certified} by bigger margins. However, tailoring and meta-tailoring can also improve a broad range of algorithms in applications outside of adversarial examples. Moreover, they could also improve these new algorithms further, as these algorithms can also be tailored and meta-tailored.

\paragraph{Compute requirements}

For the CIFAR-10 experiments building on~\citet{cohen2019certified}, each training of the meta-tailored method was done in a single GTX 2080 Ti for 6 hours. Certification was much more expensive (10k examples with 100k predictions each for a total of $10^9$ predictions). Since certifications of different images can be done in parallel, we used a cluster consisting of 8 GTX 2080 Ti, 16 Tesla V-100, and 40 K80s (which are about 5 times slower), during 36 hours.

For the ImageNet experiments, we fine-tuned the original models for 5 epochs; each took 18 hours on 1 Tesla V-100. We then used 30 Tesla V-100 for 20 hours for certification.

\end{document}